\documentclass[10pt]{article} 
\usepackage[preprint]{tmlr}

\usepackage{amsmath,amsfonts,bm}

\def\eqref#1{equation~\ref{#1}}

\def\1{\bm{1}}

\DeclareMathAlphabet{\mathsfit}{\encodingdefault}{\sfdefault}{m}{sl}
\SetMathAlphabet{\mathsfit}{bold}{\encodingdefault}{\sfdefault}{bx}{n}

\usepackage{hyperref}
\usepackage{url}

\usepackage{algorithm}
\usepackage{algorithmic}

\usepackage{multirow}
\usepackage{makecell}
\usepackage{graphicx}
\usepackage{subfigure}
\usepackage{amsmath}

\usepackage{booktabs}
\usepackage{bm}
\usepackage{amsthm}

\newtheorem{definition}{Definition} 
\newtheorem{theorem}{Theorem}

\newtheorem{hypothesis}{Hypothesis}
\newtheorem{problem}{Problem}

\newtheorem{proposition}{Proposition}

\title{SiMilarity-Enhanced Homophily for \\
Multi-View Heterophilous Graph Clustering}

\author{
\name Jianpeng Chen \email jianpengc@vt.edu \\
      \addr Virginia Polytechnic Institute and State University
      \AND
\name Yawen Ling \email yawen.ling@outlook.com \\
      \addr University of Electronic Science and Technology of China
      \AND
      \name Yazhou Ren \email yazhou.ren@uestc.edu.cn \\
      \addr University of Electronic Science and Technology of China
      \AND
      \name Zichen Wen \email Zichen.Wen@outlook.com \\
      \addr University of Electronic Science and Technology of China
      \AND
      \name Tianyi Wu \email 
      TianYi-Wu@outlook.com \\
      \addr University of Electronic Science and Technology of China
      \AND
      \name Shufei Zhang \email zhangshufei@pjlab.org.cn \\
      \addr Shanghai AI Lab
      \AND
      \name Lifang He  \email lih319@lehigh.edu
      \\
      \addr Lehigh University}

\def\month{MM}  
\def\year{YYYY} 
\def\openreview{\url{https://openreview.net/forum?id=XXXX}} 

\begin{document}

\maketitle

\begin{abstract}
With the increasing prevalence of graph-structured data, multi-view graph clustering has become a fundamental technique in various applications. While existing methods often employ a unified message passing mechanism to enhance clustering performance, this approach is less effective in heterophilous scenarios, where nodes with dissimilar features are connected. Our experiments demonstrate this by showing the degraded clustering performance as the heterophilous ratio increases.
To address this limitation, a natural method is to conduct specific graph filters for graphs with specific homophilous ratio. However, this is inappropriate for unsupervised tasks due to the unavailable labels and homophilous ratios. Alternatively, we start from an observation showing that the implicit homophilous information may exist in similarity matrices even when the graph is heterophilous.
Based on this observation, we explore a strategy that does not require prior knowledge of the homophilous or heterophilous, proposing a novel data-centric unsupervised learning framework, namely SiMilarity-enhanced Homophily for Multi-view Heterophilous Graph Clustering (SMHGC).
By analyzing the relationship between similarity and graph homophily, we propose to enhance the homophily by introducing three similarity terms, \emph{i.e.}, neighbor pattern similarity, node feature similarity, and multi-view global similarity, in a label-free manner. Then, a consensus-based inter- and intra-view fusion paradigm is proposed to fuse the improved homophilous graph from different views and utilize them for clustering. 
The state-of-the-art experimental results on both multi-view heterophilous and homophilous datasets highlight the effectiveness of using similarity for unsupervised multi-view graph learning, even in heterophilous settings. Furthermore, the consistent performance across semi-synthetic datasets with varying levels of homophily serves as further evidence of SMHGC's resilience to heterophily.
\end{abstract}

\section{Introduction}
Graph neural networks (GNNs)~\citep{scarselli2008graph, kipf2016semi, velickovic2019deep, GAE, 10.1145/3581783.3612358} that can handle graph data by considering both node features and neighbor relations show impressive performance in various domains, such as social networks~\citep{10.1145/3554981, socgnn_kdd22, usdefake_wsdm23}, recommendation systems~\citep{imix_neurips22, sht_kdd22, 10.1145/3543507.3583530} and molecules~\citep{hmgnn_icml22, molgnnpretrain_neurips22, kgpt_kdd22}. However, labeling the growing explosion of graph-structured data is often intricate and expensive. Deep multi-view graph clustering (MVGC) has recently been explored to address this problem under unsupervised settings by leveraging the complementary and consistent information of different graph views. For example, O2MAC~\citep{o2multi} captures the shared feature representation by designing a One2Multi GNN-based graph autoencoder. MVGC~\citep{XiaWYGHG22} explores the cluster structure by training a graph convolutional encoder to learn the self-expression coefficient matrix. MCGC~\citep{pan2021multi} learns a consensus graph to exploit both attribute content and graph structure information simultaneously.

Despite much progress made in this area, these methods based on GNNs typically rely on an implicit homophily assumption, \emph{i.e.}, connected nodes often belong to the same class, as pointed out by \cite{zhu2020beyond, ma2022is}. With this assumption, the message passing mechanism in GNNs can effectively aggregate the node information from the same class while disregarding the scrambled information, enabling the model to obtain class-distinguishable embedding for downstream tasks as substantiated by extensive empirical evidence~\citep{tam_icml22, gppt_kdd22, edgeformer_iclr23, 10.1145/3581783.3612469, 10.1145/3581783.3613915}. Unfortunately, the graphs collected in reality often fail to fully satisfy the homophily assumption, which significantly limits the applicability of GNN-based MVGC methods. In fact, a more common graph is moderately or even mildly homophilous rather than a fully homophilous graph due to the general presence of non-homophilous (heterophilous) information in graphs. When it comes to such heterophilous graphs, the message passing mechanism aggregates the node information from different classes, hindering access to class discriminative embeddings.
The empirical results shown in Fig.~\ref{fig:HeteRatio-NMI} demonstrate the challenge of heterophilous in the context of unsupervised clustering task, \emph{i.e.}, the steady degrades of clustering performance with the increase of heterophilous ratio. 
Unlike supervised tasks with ground truth labels where the homophilous ratio could be estimated and then improved by the observation of training data, it is more difficult to design frameworks for unsupervised clustering on homophilous ratio agnostic graphs. Therefore, one key point for MVGC is \textit{how to learn class discriminative embeddings on unsupervised tasks with unknown homophilous ratio}, which termed as multi-view heterophilous graph clustering (MVHGC).

Recently, several works have tried to address the heterophilous issue for MVGC. These works, generally, focus on two aspects. One is to utilize the pseudo labels predicted by multiple views to gradually improve the aggregated graphs, for example, DuaLGR~\citep{ling2023dual} proposes a dual pseudo label guided framework. Another type tries to adaptively extract useful information by utilizing elaborate graph filters. For example, MCGC~\citep{pan2021multi} and AHGFC~\citep{AHGFC} propose hybrid graph filters for learning node embedding. However, the first type relies on the robustness of pseudo labels, the bad quality of pseudo labels may lead to local optimums. The second type is limited due to the difficulty in determining which graph filter should be used for which specific graph, because of the unaccessible homophilous ratio.


Unlike these model-centric methods, we address this issue from a data-centric view, which could skip the limitations of previous works that rely on pseudo labels or GNN models (graph filters). Our observations on the benchmark data demonstrate the homophilous ratio of a graph could be potentially improved. As shown in Fig.~\ref{fig:neighbor_pattern_sim}, the homophilous ratio could essentially be improved by the proposed neighbor pattern similarity and feature similarity. This suggests that homophilous information can exist even in heterophilous graphs, a factor that has been overlooked by previous work.

\begin{figure} 
  \centering
  \subfigure[On basic message passing mechanism, normalized mutual information (NMI) is heavily impacted by the heterophilous ratio.]{
  \includegraphics[width=0.4\textwidth]{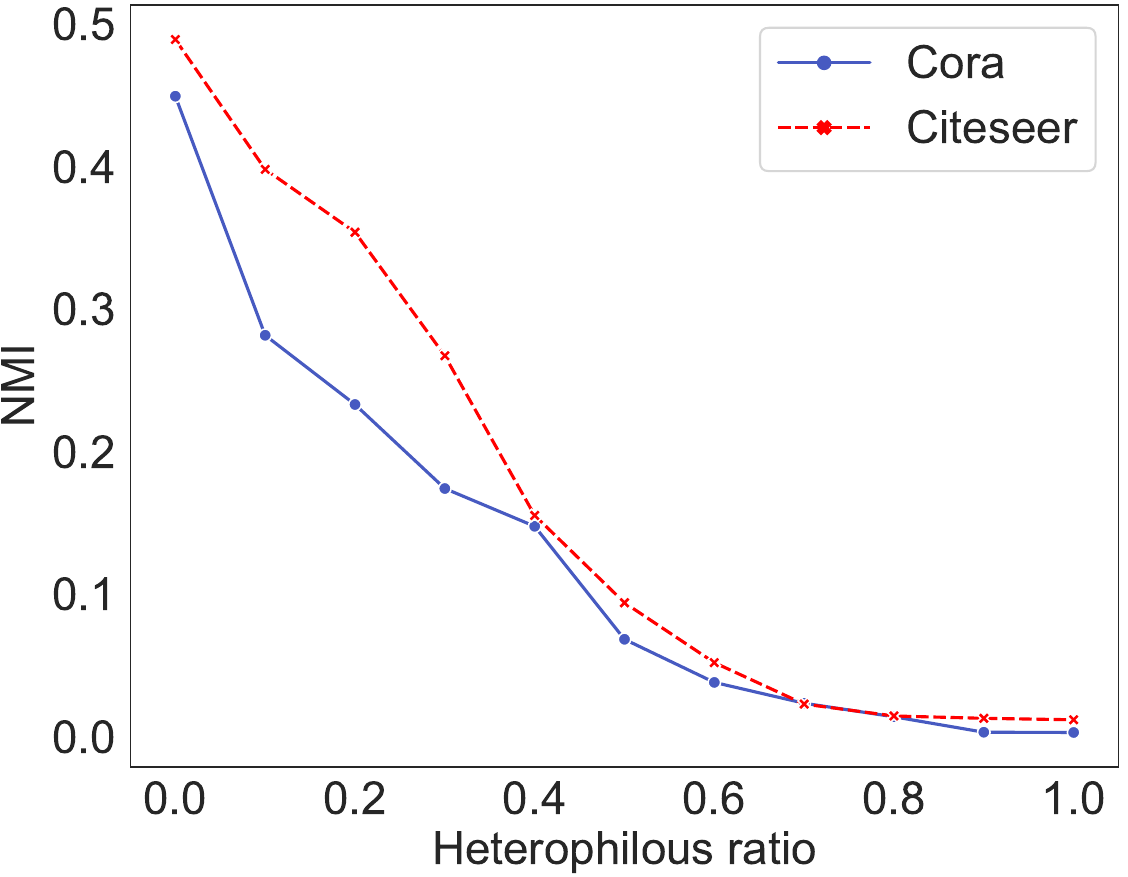} \label{fig:HeteRatio-NMI}}
  \hspace{0.1in}
  \subfigure[Homophilous ratio on homophilous graphs (ACM and DBLP) and heterophilous graphs (Texas and Chameleon). Red is the surpassed value to green.]{ 
  \includegraphics[width=0.4\textwidth, ]{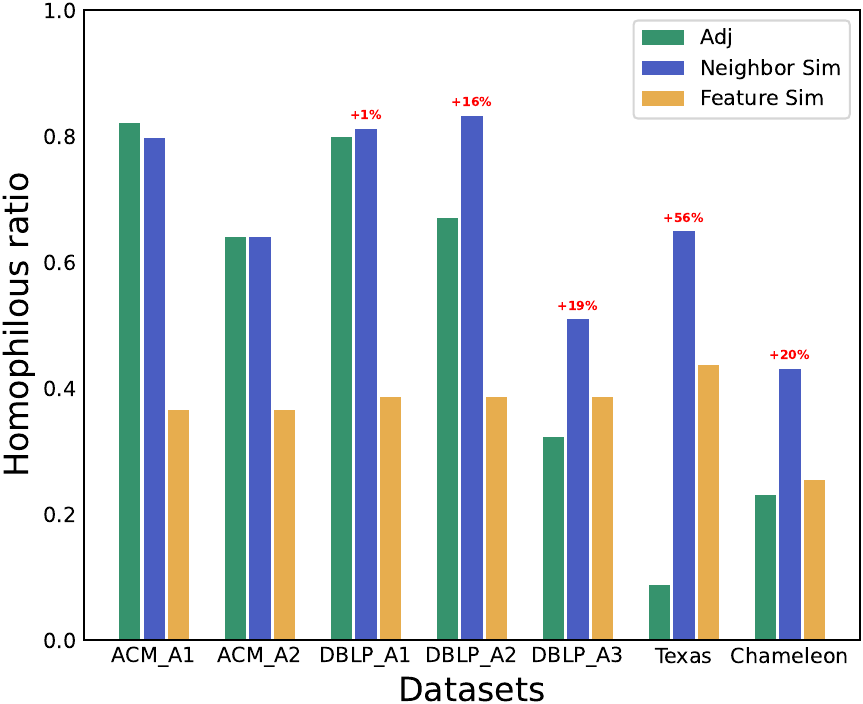} \label{fig:neighbor_pattern_sim}}
  \caption{\textbf{(a) Observation 1}: Clustering performance decrease with the increase of heteropihlous ratio.\textbf{(b) Observation 2}: On heterophilous graph (Texas and Chameleon), homophilous ratio of neighbor pattern similarity (Definition~\ref{def:neiborsim}) and feature similarity could be better than the original adjacent.} 
  \label{fig:observations}
\end{figure}

Based on the observation, we propose our solution from a significantly effective yet generally underestimated perspective - \emph{Similarity}. To avoid the limitations of GNNs on heterophilous graphs for clustering tasks, as described in Fig.~\ref{fig:observations} and Section~\ref{sec:limitGNN}, we propose three similarities, \emph{i.e.}, neighbor pattern similarity, node feature similarity, and multi-view global similarity, to effectively improve the homophily of graph. 
Building on this concept, we propose a SiMilarity-enhanced homophily Multi-view Heterophilous Graph Clustering framework (SMHGC). In this framework, a robust similarity-enhanced homophilous graph can be obtained and iteratively updated with the optimization of the global similarity as well as the backpropagation of multi-view information, ultimately enhancing the subsequent message passing process and clustering results.

The experiments show that the proposed SMHGC achieves state-of-the-art (SOTA) results on both widely used homophilous multi-view graph datasets and heterophilous graph datasets. Moreover, when considering six semi-synthetic multi-view heteraphilous graph datasets with varying heterophilous ratios,  SMHGC performs perfectly without any decrease in clustering performance even when the heterophilous ratio increases on semi-synthetic MVGC datasets. This contrasts sharply with previous studies, which showed a significant decline in clustering results. Especially on the semi-synthetic graph with heterophilous ratios greater than 70\%, SMHGC significantly improves the normalized mutual information by over 30\% compared to previous SOTAs. In addition, our ablation study further demonstrates the effectiveness of the proposed components.

The contributions of this paper can be summarized as follows:
\begin{itemize}
\item We propose three different similarities (neighbor pattern similarity, node feature similarity, and multi-view global similarity) to extract and fuse the homophilous information. The relationship between homophily and our proposed similarity terms is analyzed, and we empirically demonstrate that the similarity can efficiently extract homophilous information in a label-free manner.

\item Distinguishing from existing model-centric approaches, we propose a novel framework called SMHGC that processes heterophilous graphs before obtaining node embeddings. SMHGC can effectively utilize and optimize the homophilous information extracted by the proposed three similarity terms.

\item Extensive experiments on both homophilous and heterophilous datasets demonstrate the superior performance of SMHGC. Moreover, we experimentally validate the feasibility and effectiveness of improved unsupervised learning performance by using similarity to extract homophily.
\end{itemize}

\section{Preliminaries}
\subsection{Notations and Definitions}
Let $\mathcal{G} = (\mathcal{V}, \mathcal{E})$ be an undirected graph, where $\mathcal{V} = \{x_i\}^{N}_{i=0}$ is the node set with node numbers $N = |\mathcal{V}|$, and $\mathcal{E} = \{e_{i,j}\vert 0\leq i,j < N\}$ is the edge set with self-loop. $\mathbf{X} \in \mathbb{R}^{N \times d}$ denotes the feature matrix of the nodes, where $x_i$ is the $d$-dimensional feature vector of node $i$. $\mathbf{A} \in \mathbb{R}^{N \times N}$ is the symmetric adjacency matrix of $\mathcal{G}$, where $a_{ij} = 1$ if there exists an edge between node $i$ and node $j$, otherwise $a_{ij} = 0$. Let diagonal matrix $\mathbf{D}$ represent the degree matrix of $\mathbf{A}$, \textit{i.e.}, $\mathbf{D}_{ii} = \sum_j \mathbf{A}_{ij}$. Here we consider the normalized adjacency matrix $\mathbf{A}$ to be normailzed as $\tilde{\mathbf{A}} = \mathbf{D}^{-1}\mathbf{A}$. In the setting of MVGC, nodes and the relations among them can be seen from multiple views. Specifically, given $V$ different views with graphs $\{\mathcal{G}^v\}_{v=1}^V$ and features $\{\mathbf{X}^v\}_{v=1}^V$, the goal of MVGC is to partition the nodes into different classes.

\paragraph{\textbf{Homophily and heterophily.}}
Different from homogeneity/ heterogeneity which describes the types of nodes and edges in a graph, homophily/ heterophily is a concept that describes the nature of edges in a graph~\citep{mcpherson2001birds, zheng2022graph}. To make the distinction easier, we formally define them again as follows.
\begin{definition}[Homophily and heterophily]\label{d1}
Let $\mathbf{Y}$ be the set of node labels, where $y_i$ denotes the label of node $i$. Let $e_{i,j}$ be an arbitrary edge connecting node $i$ and node $j$ in $\mathcal{V}$. Homophily describes $y_i = y_j$ for the edge $e_{i,j}$, and, conversely, it is termed heterophily.
\end{definition}

For a graph, it consists of only homophilous and non-homophilous edges (heterophilous edges). Based on this, we can define the homophilous information in a graph with the help of generalized edge as follows:
\begin{definition}[Generalized edge]\label{def:edge}
The generalized edge set $\widetilde{\mathbf{E}}$ is defined by transforming the inputs $\mathbf{X}$ and $\mathbf{A}$ through any transformation operation $\text{Tf}(\cdot)$ (\emph{e.g.}, initial inputs or multi-order aggregation): $\widetilde{\mathbf{E}} = \text{Tf}(\mathbf{X}, \mathbf{A}): \mathbb{R}^{N \times d}, \mathbb{R}^{N \times N} \rightarrow \mathbb{R}^{N \times N}$. Each element $\widetilde{e}_{i,j}$ in $\widetilde{\mathbf{E}}$ is a generalized edge that connects nodes $i$ and $j$.
\end{definition}

\begin{definition}[Homophilous information]\label{d2}
For an arbitrary generalized edge $\widetilde{e}_{i,j} \in \widetilde{\mathbf{E}}$, it exhibits homophily if and only if $y_i = y_j$, \textit{i.e.}, nodes $i$ and $j$ connected by it belong to the same cluster. The homophilous information is defined as the set of homophilous generalized edges.
\end{definition}

In recent works~\citep{zhu2020beyond, lim2021large}, the ratio of homophilous edges, named homophily ratio, is used to measure the homophily of a given graph. The graph becomes strongly homophilous when this ratio closes to $1$, and conversely, the graph is with strong heterophily (\textit{i.e.}, weak homophily) when the ratio closes to $0$. However, it is notable that this metric can only be used for evaluation but cannot contribute to the learning of models in unsupervised scenarios since this process relies on labels that are agnostic.

With the aforementioned notations, the problem of MVHGC can be formalized as follows.
\begin{problem}[Multi-view heterophilous graph clustering]
\textbf{\\Input:} Graphs $\{\mathcal{G}^v\}^V_{v=1}$ encompassing both homophilous and heterophilous information and corresponding node feature matrices $\{\mathbf{X}^v\}_{v=1}^V$ from $V$ views.
\textbf{\\output:} The clustered node labels for the total $N$ nodes.
\end{problem}

\subsection{Limitation of GNNs for Clustering on Heterophilous Graph}
\label{sec:limitGNN}
Generally, contemporary GNNs rely on a message passing mechanism in which each node updates itself by aggregating the embedding of its neighboring nodes and combining it with its own embedding~\citep{xu2018powerful}. Specifically, the embedding update process of node $i$ at the $l$-th GNN layer can be expressed as:
\begin{equation}\label{gnn}
    m_i^{l} = \text{AGGREGATE}^l (\{h_u^{l-1} | u \in \mathcal{N}(i)\}); 
     h_i^l = \text{UPDATE}^l (h_i^{l-1}; m_i^l), 
\end{equation}

where $h_i^l$ denotes the embedding of node $i$ at the $l$-th layer, $\mathcal{N}(i)$ represents the neighborhood of node $i$ and $m_i$ represents the messages aggregated from the neighborhood of node $i$. $\text{AGGREGATE}(\cdot)$ and $\text{UPDATE}(\cdot; \cdot)$ are defined as aggregation and update operations in the forward process. From Eq.~(\ref{gnn}), it can be seen that messages from neighborhood aggregation play a crucial role in the update of node $i$'s embedding. In homophilous neighborhoods, messages are clear and come from pure neighbors of the same class. While in heterophilous neighborhoods, mixed-class information weakens embedding distinguishability. We claim that this distinguishability also exists in unsupervised clustering task. Fig.~\ref{fig:HeteRatio-NMI} demonstrates this by performing a basic two-order nonparametric message passing on two single-view datasets (Cora and Citeseer) respectively, showing that performance is greatly affected by heterophilous information. 
Therefore, directly aggregating messages from neighbors through a heterophilous graph may not be an effective choice.

\section{Proposed Framework}
\subsection{Overview of SMHGC}
In this section, we present the SMHGC framework, focusing on three key similarities. Fig.~\ref{fig:overview} illustrates an overview of SMHGC. $V$ feature matrices ($\{\mathbf{X}^v\}^V_{v=1}$) and corresponding $V$ graphs ($\{\mathbf{A}^v\}^V_{v=1}$) are input. Then, in the homophilous information extraction module, the feature matrices and adjacent matrices are transformed into two similarity subspaces $\mathbf{Z}_a$ and $\mathbf{Z}_x$ respectively to construct two similarities, \emph{i.e.}, neighbor pattern similarity ($\mathbf{A}_a$) and node feature similarity ($\mathbf{A}_x$) (Section~\ref{par:CSHIE}). After that, in the intra-view fusion module, the global similarity matrix $\overline{\mathbf{H}}\overline{\mathbf{H}}^{\mathrm{T}}$ which implies multi-view consistency similarity information is introduced to measure and weight the aforementioned two similarities, so that the graph $\mathbf{S}^v$ obtained from the three similarities is able to absorb not only the homophilous information from neighbor pattern and node feature but also the multi-view consistency similarity information from different views. This extracted homophilous structure is then aggregated with node features to get intra-view embeddings (Section~\ref{par:intraFusion}). Finally, multiple view-specific embeddings are fused under the guidance of consensus information to produce a better global embedding, $\overline{\mathbf{H}}$. Iteratively, the updated $\overline{\mathbf{H}}$ further aids in the generation and optimization of view-specific similarity graph $\mathbf{S}$ (Section~\ref{par:interFusion}). Finally, the optimized global embedding is fed into traditional clustering methods (\emph{e.g.}, $K$-means) to obtain clustering results.
\begin{figure*}[tb]
    \centering
    \includegraphics[width=1\linewidth]{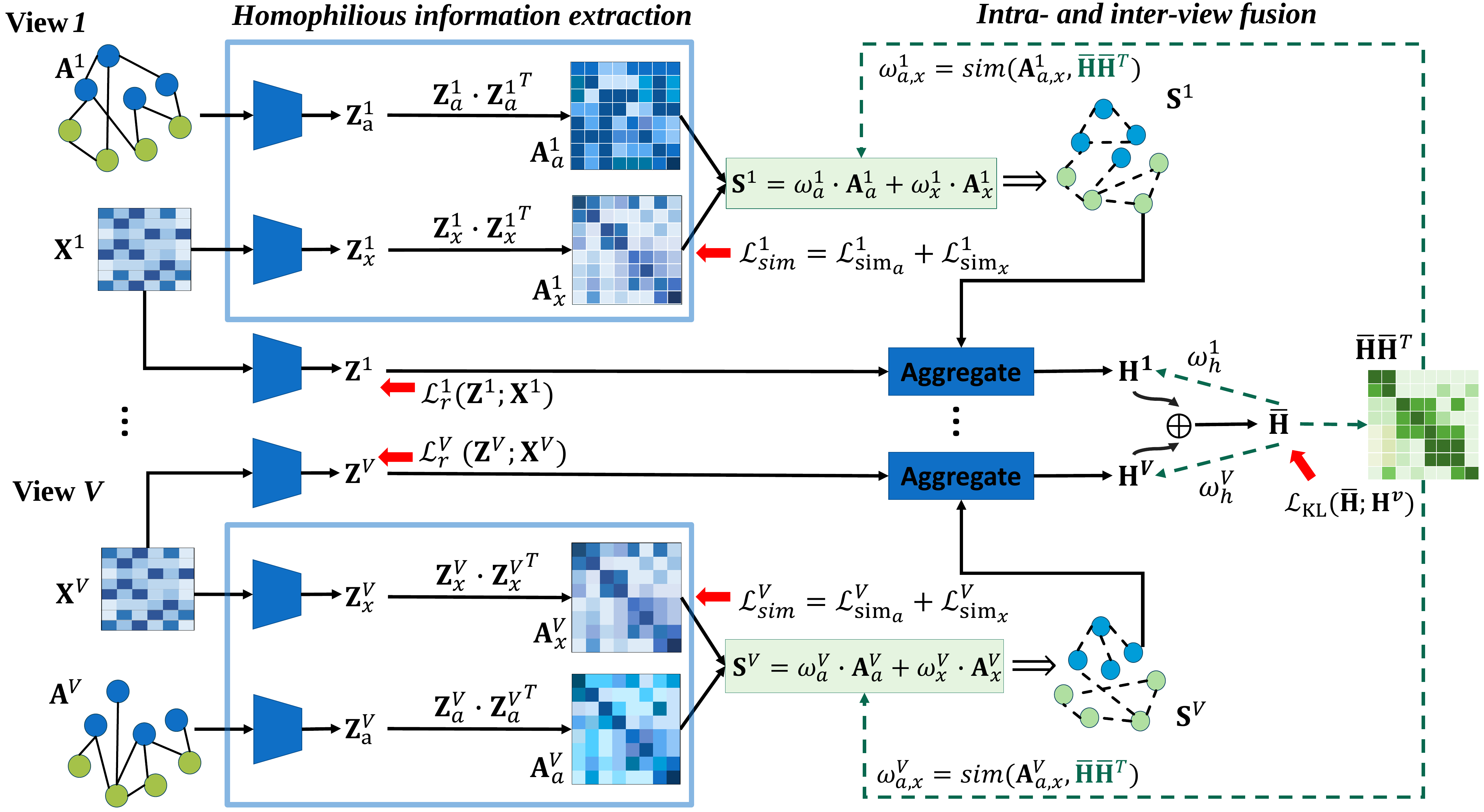}
    \caption{The proposed framework of SMHGC. It takes $V$ feature matrices ($\{\mathbf{X}^v\}^V_{v=1}$) and corresponding $V$ graphs ($\{\mathbf{A}^v\}^V_{v=1}$) as inputs. Then, homophilous information is extracted and optimized under the regularization of similarity terms ($\{\mathcal{L}^v_{sim}\}^V_{v=1}$). The homophilous graphs ($\{\mathbf{S}^v\}^V_{v=1}$) are then generated by infusing the homophilous information from both neighbor similarity matrix ($\mathbf{A}^v_a$) and feature similarity matrix ($\mathbf{A}^v_x$). Subsequently, the intra- and inter-view fusion module aggregates and fuses feature and homophilous information together to finally output a comprehensive embedding $\overline{\mathbf{H}}$.}
    \label{fig:overview}
\end{figure*}

Overall, SMHGC utilizes similarity to extract and enhance the homophilous information that is implied in features and heterophilous graphs, and then they are weighted and fused under the guidance of the multi-view global similarity. Therefore, a similarity graph $\mathbf{S}^v$ with comprehensive homophilous information can be obtained, benefiting the following message passing process as demonstrated by the practical experiments in Fig.~\ref{fig:HeteRatio-NMI}, \emph{i.e.}, more homophilous information results in better clustering results after message passing.

\subsection{Couple Similarity Enhanced Homophily}
\label{par:CSHIE}
\paragraph{\textbf{Neighbor pattern similarity for homophilous information extraction.}}
For a graph, it consists of only homophilous and heterophilous information. Instead of directly using this graph to aggregate neighbor messages, our motivation is to extract the implied homophilous information from this graph, so the negative effect of heterophilous information can be effectively avoided. A natural question is \textit{what is the homophilous information in a heterophilous graph?}

To answer the question, we investigate various prior studies. For example, \citet{he2023the} demonstrated that neighbor patterns are the key factor for GNNs to improve performance. Moreover, some research suggests that `good heterophily', where the nodes with the same label sharing similar neighbor patterns, can be exploited to achieve good performance~\citep{song2023ordered, ma2022is}.
These research provides us a support and possible solution to find homophilous information implied in heterophilous graphs. Following these previous works, we summarize their observation in the following Proposition~\ref{good}:
\begin{proposition} [Good heterophily~\citep{ma2022is}]\label{good}
    In heterophilous graphs, if the neighborhood distribution of nodes with the same label is (approximately) sampled from a similar distribution and different labels have distinguishable distributions, then this heterophilous graph indicates good heteropihly.
\end{proposition}

Regrettably, this notion of `good heterophily', as inferred from the label information, is unsuitable for unsupervised MVHGC task.
To generalize this proposition to unsupervised setting, we propose the hypothesis:
\begin{hypothesis} \label{hyp:sim_homo}
In node clustering tasks, the higher the similarity between two nodes in an ideal subspace, the greater the likelihood that they belong to the same cluster.
\end{hypothesis}

With Hypothesis~\ref{hyp:sim_homo}, the labels mentioned in Proposition~\ref{good} can be generalized to similarity, which implies that the nodes with similar neighbor patterns in the ideal subspace are homophilous nodes, so that to be used in the clustering task. More importantly, the target of learning homophilous information from a good heterophilous graph can be naturally changed to the target of learning a better projection that can project the input nodes to the ideal subspace. 
Therefore, it becomes achievable to find the homophilous information implied in good heterophilous graphs by considering the neighbor pattern similarity:
\begin{definition}[Neighbor pattern similarity]\label{def:neiborsim}
Given adjacency matrix $\mathbf{A}$, each row in $\mathbf{A}$ describes the neighbor pattern of the node. Therefore, we define $\mathbf{A}\mathbf{A}^\mathrm{T}$ as neighbor pattern similarity.
\end{definition}

For an intuitive understanding, from the neighborhood aggregation perspective, neighbor pattern similarity implies 2-hop structural information, which provides a possibility to receive more homophilous edges compared to the initial heterophilous adjacency matrix.
The experimental results shown in Fig.~\ref{fig:sim_hr} also demonstrate that the neighbor pattern similarity may capture homophilous information implied in heterophilous graph.

Following the analysis, for acquiring homophilous information, we employ deep encoders to initially conduct subspace learning for node neighbors. This process is then refined through multi-view consensus, facilitated by a regularization term in Eq.~(\ref{Lsim}) that captures neighbor pattern similarity.
Specifically, each row of the given graph $\mathbf{A}$ implies a neighbor pattern of a node, so $\mathbf{A}$ is fed into an encoder $f_a$ instantiated with a multi-layer perception (MLP) with learnable parameters $\theta_a$, and then outputs a low-dimensional representation: 
\begin{equation}
    \mathbf{Z}_a = f_{a} (\mathbf{A}; \theta_a),
\end{equation}
where $\mathbf{Z}_a$ aims to represent the potential neighbor pattern in ideal subspace.

Subsequently, we propose the neighbor pattern similarity regularization term $\mathcal{L}_{sim_a}$ to encourage the encoder $f_a$ to focus on learning neighbor pattern similarity information of the input graph:
\begin{equation}
    \mathcal{L}_{sim_a} = l_s(\mathbf{Z}_a{\mathbf{Z}_a^\mathrm{T}}; \mathbf{A}\mathbf{A}^\mathrm{T}),
\end{equation}
where $l_s(\cdot; \cdot)$ is loss calculation function, instantiated here as the Mean Squared Error (MSE) loss. In this similarity loss, $\mathbf{A}\mathbf{A}^\mathrm{T}$ is the neighbor pattern similarity information of the input graph. Consequently, by utilizing this loss, we aim to guide the encoder $f_a$ to learn as much information as possible regarding neighbor pattern similarity among nodes, thereby capturing the potential homophily within the given heterophilous graph. 

\paragraph{\textbf{Node feature similarity for homophilous information extraction.}}
On the other hand, the node feature $\mathbf{X}$ describes the inherent properties of the nodes, and these properties can also reflect the homophilous information between the nodes. Therefore, in addition to relying on the neighbor relations within $\mathcal{G}$, more homophilous information can be explored through the similarity information within the node features. Similar to neighbor pattern similarity, an encoder $f_x(\mathbf{X}; \theta_x) = \mathbf{Z}_x$ is trained using the node feature similarity regularization term $\mathcal{L}_{sim_x}$: 
\begin{equation}
    \mathcal{L}_{sim_x} = l_s(\mathbf{Z}_x{\mathbf{Z}_x^\mathrm{T}}; \mathbf{X}\mathbf{X}^\mathrm{T}).
\end{equation}

Finally, the couple similarity-based regularization loss for extracting homophilous information can be expressed as:
\begin{equation}\label{Lsim}
    \mathcal{L}_{sim} = \mathcal{L}_{sim_a} + \mathcal{L}_{sim_x} 
    = l_s(\mathbf{Z}_a{\mathbf{Z}_a^\mathrm{T}}; \mathbf{A}\mathbf{A}^\mathrm{T}) + l_s(\mathbf{Z}_x{\mathbf{Z}_x^\mathrm{T}}; \mathbf{X}\mathbf{X}^\mathrm{T}).
\end{equation}

\paragraph{\textbf{How the couple similarity enhances clustering performance.}}
The neighbor pattern similarity defined in Definition~\ref{def:neiborsim} enables the extraction of both homophilous information and 'good heterophily' from the input graph in a label-free manner. On the one hand, it involves a 2-hop aggregation of the adjacency matrix, allowing for the comprehensive mining of homophilous information within it. On the other hand, through similarity computation, similar neighboring patterns can be efficiently explored, facilitating the identification of `good heterophily'. In addition, feature similarity extracts homophilous information from feature space. 

\begin{figure} 
  \centering
  \includegraphics[width=0.6\textwidth]{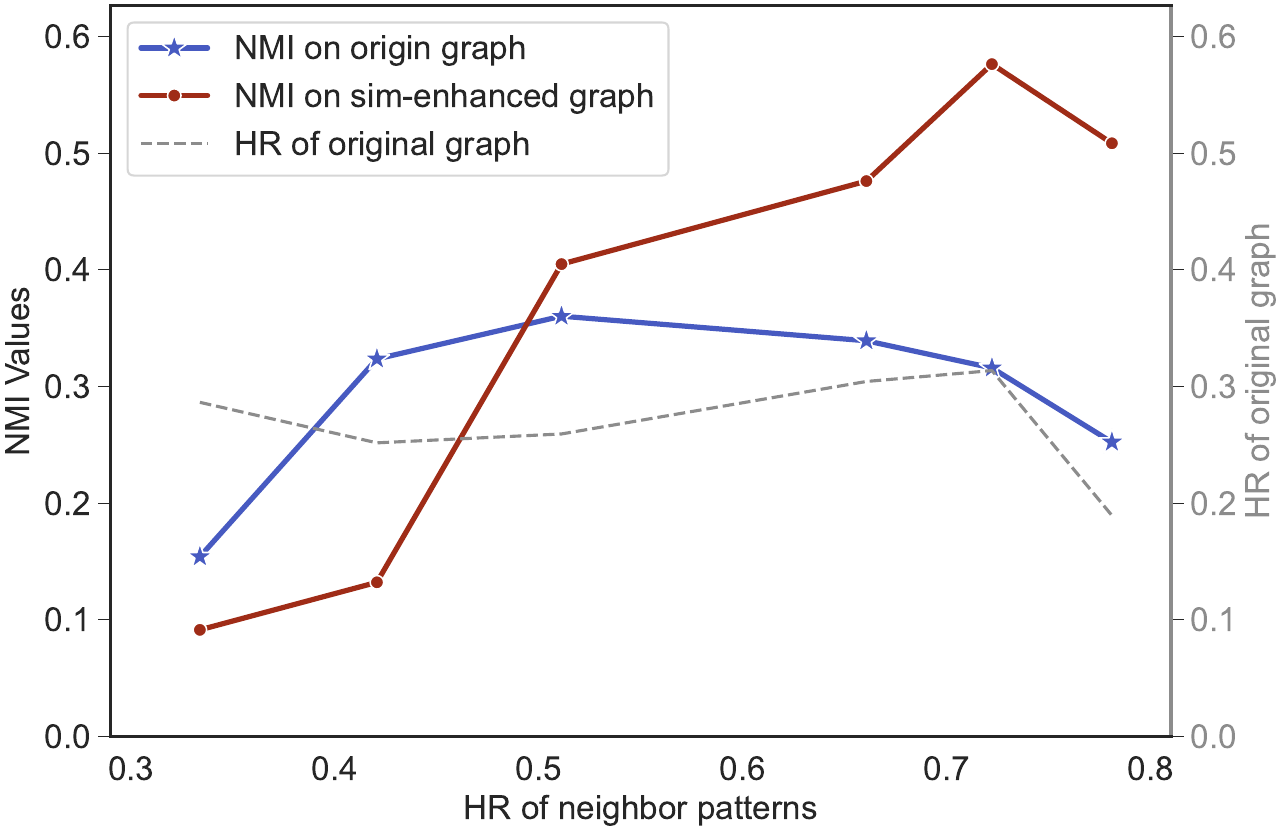} 
  \caption{Left of Y-axis is the clustering performance (NMI). X-axis denotes synthesis datasets with increased `good heterophily' constructed from ACM. To construct the synthesis dataset, the homophilous ratio of the original graph is kept low (as the gray dotted line shows), instead, we gradually increase `good heterophily' (HR of neighbor patterns) information. The original graph and sim-enhanced graph are fed into a parameter-free message passing layer to get aggregated node embedding, and $K$-means is conducted to obtain clusters evaluated by NMI.} 
  \label{fig:sim_hr}
\end{figure}

In this subsection, we empirically explore how the neighbor pattern similarity and feature similarity enhance the clustering performance, therefore substantiating Hypothesis~\ref{hyp:sim_homo} in the meantime. 
Fig.~\ref{fig:HeteRatio-NMI} suggests that the message passing on heterophilous graph leads to poor clustering performance. Furthermore, as the blue line depicted in Fig.~\ref{fig:sim_hr}, the increase of homophilous information in a heterophilous graph still cannot increase the clustering performance via a message passing layer. In comparison, we construct an enhanced graph from couple similarity (sim-enhanced graph, denoted as $S$ below), which directly combines the information obtained from neighbor pattern similarity and feature similarity via weighted sum:
$$
    S = \omega_x \mathbf{X}\mathbf{X}^\mathrm{T} + \omega_a\mathbf{A}\mathbf{A}^\mathrm{T},
$$
where $\mathbf{X}\mathbf{X}^\mathrm{T}$ and $\mathbf{A}\mathbf{A}^\mathrm{T}$ are the targets of our proposed regularization terms, and $\omega_x$ and $\omega_a$ are set as the homophilous ratio of feature similarity and neighbor pattern similarity respectively. The results in Fig.~\ref{fig:sim_hr} are shown in red line: the increase of `good heterophily' information leads to better clustering performance on sim-enhanced graph as inputs of a message passing layer.

This observation precisely explains why the proposed SMHGC maintains superior clustering performance shown in Fig.~\ref{fig:resultsheternmi} and Fig.~\ref{fig:resultsheteracc}. The results presented in Fig.~\ref{fig:sim_hr} suggest that \textit{the couple similarity enhanced graph could capture richer homophilous information from heterophilous graph, including `good heterophily' (denoted as HR of neighbor patterns), compared to the initial input, leading to better clustering performance.}

\subsection{Global Similarity Guided Intra-View Homophily Fusion and Aggregation}
\label{par:intraFusion}
In unsupervised tasks, the inaccessibility of labels and homophily may lead to a lack of robustness when extracting homophilous information. Therefore, in this section, we address the challenge that \textit{how to obtain more robust homophilous information}, and seek the guidance of consensus from different views.
\paragraph{How to fuse node feature and neighbor pattern similarities.} The node features and neighbor patterns in each view's graph $\mathcal{G}^v$ may imply different homophilous information, which contributes differently to the downstream task. To generate a robust homophilous graph, we design a global similarity guided intra-view homophily fusion strategy, aiming to assign weights to the extracted homophilous information based on their relevance to the final task and multi-view global similarity.
Without labels, it is difficult to directly determine the contribution of information to the final task. An alternative approach is inspired by clustering algorithms like $K$-means~\citep{macqueen1965some}, where the distance between samples and their cluster centroids dictates their final assignments. In other words, the distance between samples plays a basis role for final assignments. Similarly, the fused embedding from all views (denoted as $\overline{\mathbf{H}}$), which is used for final assignments, captures the basis that is most aligned with the downstream tasks. 

Therefore, it is natural to use the global similarity matrix obtained from $\overline{\mathbf{H}}$ to evaluate the homophilous information of node features and neighbor patterns. 
Specifically, let the homophilous information from node features and neighbor patterns be denoted as $\mathbf{A}_x = \mathbf{Z}_x {\mathbf{Z}_x^\mathrm{T}}$ and $\mathbf{A}_a = \mathbf{Z}_a {\mathbf{Z}_a^\mathrm{T}}$, respectively. Their contribution to the task objective can be assessed by evaluating their similarity to $\overline{\mathbf{H}}$:
\begin{equation} \label{eq:fusion}
    (\omega_x, \omega_a) = \text{norm}(sim(\mathbf{A}_x; \overline{\mathbf{H}}\overline{\mathbf{H}}^\mathrm{T}), sim(\mathbf{A}_a; \overline{\mathbf{H}}\overline{\mathbf{H}}^\mathrm{T})),
\end{equation}
where $\text{norm}(\cdot)$ denotes normalization operation, $sim(\cdot;\cdot)$ denotes the similarity calculation function, which is instantiated as cosine similarity in this work, and $\overline{\mathbf{H}}\overline{\mathbf{H}}^\mathrm{T}$ is the global similarity matrix. Based on this, $\mathbf{A}_x^v$ and $\mathbf{A}_a^v$ in $v$-th view can be properly fused to generate a homophilous graph as follows:
\begin{equation}
    \mathbf{S}^v = \omega_x^v \mathbf{A}_x^v + \omega_a^v \mathbf{A}_a^v.
\end{equation}

To accommodate the discrete nature of graph, we discretize the dense $\mathbf{S}^v$. Specifically, let $U_i^v$ be the set of the top $k$ largest elements in $\mathbf{S}^v_{i,:}$, where $k$ is a hyperparameter, thus $\mathbf{S}^v$ can be discretized as:
\begin{equation} \label{eq:sdis}
    s^v_{ij}=
    \left\{
        \begin{aligned}
        &1, \quad \text{if} \quad s_{ij}^v \in {U}_i^v, \\
        &0, \quad \text{otherwise}.
        \end{aligned}
    \right.
\end{equation}

\paragraph{Aggregate node feature and extracted graph via GNN.} Compared to the neighbor relations in a graph composed of solely homophilous and heterophilous edges, the node features encompass a variety of information that can be leveraged for obtaining distinguishable embeddings. Therefore, the node features are required to be compressed in a latent space. In this work, autoencoders are employed to compress the inherent distinguishability of these node features.
Let $f^v_{\phi^v}$ and $g^v_{\xi^v}$ be the encoder and decoder with parameters $\phi^v$ and $\xi^v$, respectively. Then, the node feature $\mathbf{X}^v$ in the $v$-th view can be reconstructed as $\hat{\mathbf{X}}^v = g^v_{\xi^v}(\mathbf{Z}_f^v) = g^v_{\xi^v}(f^v_{\phi^v}(\mathbf{X}^v))$, where $\mathbf{Z}_f^v$ is the latent representation. Based on this, the reconstruction loss function used to train the autoencoder is:
\begin{equation}\label{Lre}
    \mathcal{L}^v_{r} = l_{r}(\hat{\mathbf{X}}^v; \mathbf{X}^v) = l_{r}(g^v_{\xi^v}(f^v_{\phi^v}(\mathbf{X}^v)); \mathbf{X}^v),
\end{equation}
where $l_{r}(\cdot; \cdot)$ is the loss function, which can be instantiated as a cross-entropy loss.

Given the construction of the graph $\mathbf{S}^v$ by associating homophilous information and inspired by~\cite{pmlr-v97-wu19e}, we opt to eliminate redundant parameters in the graph convolution process. Instead, we directly convolve $\mathbf{S}^v$ with node feature $\mathbf{Z}_f^v$, which not only reduces model complexity but also enhances its generalization ability. Furthermore, we implement the aggregation in the form of residuals, preserving information for each order of operation. This simplified GNN is expressed as:
\begin{equation}\label{eq:aggre}
\begin{aligned}
    & \mathbf{H}^v = \mathbf{h}^{v,0} + \mathbf{h}^{v,1} + \cdots + \mathbf{h}^{v, order},\\
    & \mathbf{h}^{v,order} = (\mathbf{S}^v)^{order}\mathbf{Z}_f^v,
\end{aligned}
\end{equation}
where $\mathbf{H}^v$ is the embedding of the $v$-th view, $\mathbf{h}^{v,0} = \mathbf{Z}_f^v$, and $order$ is a hyperparameter that is designed to control the order of aggregation. Note that the conducted aggregation in Eq.~(\ref{eq:aggre}) can be replaced by any complicated GNNs.

\subsection{Global Consensus Based Inter-View Fusion}
\label{par:interFusion}
In multi-view tasks, each view can contribute different information to the downstream task to varying degrees. To fully exploit the complementarity of the individual views, it is natural to fuse the embeddings from all views. However, the quality of information may vary across views, necessitating the assignment of appropriate weights to each view during fusion. Given the absence of label information, an alternative approach is to use inter-view consensus for evaluation. Specifically, we utilize the fused embedding $\overline{\mathbf{H}}$ from the current iteration to evaluate and score the embedding $\mathbf{H}^v$ obtained from each view. Based on this, we determine the weight of each view for fusion. Ultimately, the weighted fused embedding is updated: 
\begin{equation} \label{eq:ovaH}
\begin{aligned}
    &\overline{\mathbf{H}} = \sum_{v=1}^V \omega_h^v \mathbf{H}^v,
    \omega_h^v = (\frac{score^v}{\max{(score^1, score^2, \cdots, score^V)}})^\rho,
\end{aligned}
\end{equation}
where $score^v$ can be obtained from the evaluation function, $score^v = metric(\mathbf{H}^v;\overline{\mathbf{H}})$ and $metric(\cdot;\cdot)$ is instantiated as cosine similarity function in this work. The hyperparameter $\rho$ is applied to adjust the smoothing or sharpening of the view weights. Notably, the update of $\overline{\mathbf{H}}$ will in turn aid the generation of view-specific similarity graph $\mathbf{S}^v$ as the iteration proceeds. With $\overline{\mathbf{H}}$, sample clustering algorithms, such as $K$-means, can be implemented to obtain the final assignment.

\subsection{Objective Function}
Overall, the objective function of SMHGC comprises three parts: similarity regularization loss term of $V$ views $\sum_V\mathcal{L}_{sim}^v$ (Eq.~(\ref{Lsim})), reconstruction loss of $V$ views $\sum_V\mathcal{L}_r^v$ (Eq.~(\ref{Lre})), and Kullback-Leibler (KL) divergence loss $\mathcal{L}_{kl}$:
\begin{equation} \label{eq:loss}
    \mathcal{L} = \gamma_{sim} \sum_V \mathcal{L}_{sim}^v + \gamma_{r} \sum_V \mathcal{L}_r^v + \mathcal{L}_{kl},
\end{equation}
where $\gamma_{sim}$ and $\gamma_{r}$ are trade-off parameters.
The KL divergence loss ($\mathcal{L}_{kl}$) is a commonly used clustering loss applied to facilitate the model to obtain consensus embedding following the previous work~\citep{ren2022deep}. Specifically, let $q_{ij}^v \in \mathbf{Q}^v$ be the soft cluster assignment that describes the possibility of node $i$ belonging to cluster centroid $j$ in $v$-th view and $\mathbf{P}$ be the sharpen target distribution, the loss can be expressed as:
\begin{equation}\label{Lkl}
    \mathcal{L}_{kl} = \text{KL}(\overline{\mathbf{P}} \Vert \overline{\mathbf{Q}}) + \sum_{v=1}^V \text{KL}(\overline{\mathbf{P}} \Vert \mathbf{Q}^v),
\end{equation}
where $\overline{\mathbf{Q}}$ and $\overline{\mathbf{P}}$ represent the soft and target distribution of $\overline{\mathbf{H}}$, respectively. With Eq.~(\ref{Lkl}), we expect the first term to enhance the discriminability of the final embedding $\overline{\mathbf{H}}$, followed by the second term to encourage the soft distribution of each view to fit the distribution of $\overline{\mathbf{H}}$, and thus to exploit the inter-view consistency.

\subsection{Generalization Analysis}
To assess the robustness and generalizability of the proposed similarity terms, we conduct a generalization analysis.
As a way to infer the generalization bound of the similarity, we introduce the hypothesis function $h_a$: $\mathcal{A} \rightarrow \mathbb{R}^{d_a}$ and $h_x$: $\mathcal{X} \rightarrow \mathbb{R}^{d_x}$ to map the neighbor pattern and node feature into the embedding points, where $\mathcal{A}$ and $\mathcal{X}$ represent the neighbor pattern and node feature space respectively. Then the embedding points from the neighbor pattern and node feature can be obtained by $z_{a_i}:= h_a(a_i)$ and $z_{x_i}:= h_x(x_i)$. Suppose $\mathcal{H}$ represents the family of $h$. On the basis of this, the similarity loss can be expressed as:
\begin{equation}
\begin{aligned}
    &\mathcal{L}_{sim_a} =  \Vert \mathbf{Z}_a{\mathbf{Z}_a^\mathrm{T}} - \mathbf{A}\mathbf{A}^\mathrm{T} \Vert^2 
     = \frac{1}{N} \sum_{i=1}^N \sum_{j=1}^{N} \Vert h_a(a_i)h_a(a_j)^\mathrm{T} - a_i a_j^\mathrm{T} \Vert^2, \\
    &\mathcal{L}_{sim_x} =  \Vert \mathbf{Z}_x{\mathbf{Z}_x^\mathrm{T}} - \mathbf{X}\mathbf{X}^\mathrm{T} \Vert^2 
     = \frac{1}{N} \sum_{i=1}^N \sum_{j=1}^{N} \Vert h_x(x_i)h_x(x_j)^\mathrm{T} - x_i x_j^\mathrm{T} \Vert^2.
\end{aligned}
\end{equation}

For the similarity loss from neighbor patterns $\mathcal{L}_{sim_a}$, the empirical risk can be defined as:
\begin{equation}
    \hat{\mathcal{L}}_a(h) = \frac{1}{N} \sum_{i=1}^N \sum_{j=1}^{N} \Vert h_a(a_i)h_a(a_j)^\mathrm{T} - a_i a_j^\mathrm{T} \Vert^2.
\end{equation}

\begin{theorem}\label{thm2}
Let $\mathcal{L}(h)$ be the expectation of $\hat{\mathcal{L}}(h)$. Suppose that for any $a \in \mathcal{A}$ and $h \in \mathcal{H}$, there exists $M < \infty$ such that $\Vert aa^\mathrm{T} \Vert$, $\Vert h_a(a)h_a(a)^\mathrm{T} \Vert \in [0, M]$ hold. Then with probability $1-\tau$ for any $h \in \mathcal{H}$ the inequality holds:
\begin{equation}
    \hat{\mathcal{L}}(h) \leq \mathcal{L}(h) + 2\sqrt{2}M^2\sqrt{N}(4 + 3\sqrt{\log\frac{1}{\tau}}).
\end{equation}
\end{theorem}

Using Theorem~\ref{thm2}, it can be easily found that the expected risk of the similarity is bounded by the empirical risk on input neighbor patterns and node features and the complexity term that depends on $M$ and $N$. The proof and model complexity analysis are presented in the Appendix.

\section{EXPERIMENTS}
\subsection{Experimental Settings}
\paragraph{\textbf{Datasets.}}
To evaluate the effect of SMHGC, we conduct extensive experiments on ten datasets, including two real-world homophilous graph datasets (ACM\footnote{https://dl.acm.org/} and DBLP\footnote{https://dblp.uni-trier.de/}), two real-world heterophilous graph datasets (Texas\footnote{http://www.cs.cmu.edu/afs/cs.cmu.edu/project/theo-11/www/wwkb} and Chameleon~\citep{rozemberczki2021multi}) and six semi-synthetic graph datasets generated from ACM~\citep{ling2023dual}. More details about the datasets and implementation details can be found in the Appendix. The implementation of SMHGC will be published.

\paragraph{\textbf{Comparison methods.}}
The following baselines are considered: VGAE~\citep{GAE} is a single-view method. O2MAC~\citep{o2multi}, MvAGC~\citep{lin2021graph}, MCGC~\citep{pan2021multi}, MVGC~\citep{XiaWYGHG22},  DuaLGR~\citep{ling2023dual} and BTGF~\citep{qian2024upper} are six deep MVGC methods.
The results from VGAE and O2MAC on ACM and DBLP are obtained from the best in the literature, the others are conducted five times and the average results with standard deviations are reported in Table~\ref{tab:overall_results}. For MVGC~\citep{XiaWYGHG22}, we use the original version without feature augmentation technique for fair comparison.

\paragraph{\textbf{Evaluation metrics.}}
In order to evaluate the final clustering performance, this study adopts four commonly used metrics, including normalized mutual information (NMI), adjusted rand index (ARI), accuracy (ACC) and F1-score (F1), following previous works~\citep{pan2021multi, XiaWYGHG22, chen2022variational}.

\begin{table*}[t]
\small
    \centering
    \caption{The clustering results of SMHGC on two homophilous graph datasets and two heterophilous graph datasets. The best results are highlighted in bold.} 
    \label{tab:overall_results}
    \scalebox{0.9}{
    \begin{tabular}{r|cccc|cccc}
    \toprule[1.5pt]
    \multirow{2}*{Methods / Datasets} & \multicolumn{4}{c|}{ACM ($hr$ $0.82$ \& $0.64$)} & \multicolumn{4}{c}{DBLP ($hr$ $0.87$ \& $0.67$ \& $0.32$)} \\
         & NMI\% & ARI\% & ACC\% & F1\% & NMI\% & ARI\% & ACC\% & F1\% \\
    \midrule
    VGAE (\citeyear{GAE}) & $49.1$ & $54.4$ & $82.2$ & $82.3$ & $69.3$ & $74.1$ & $88.6$ & $87.4$ \\
    O2MAC (\citeyear{o2multi}) & $69.2$ & $73.9$ & $90.4$ & $90.5$ & $72.9$ & $77.8$ & $90.7$ & $90.1$ \\
    MvAGC (\citeyear{lin2021graph}) & $57.8$ & $60.1$ & $84.4$ & $84.6$ & $67.9$ & $74.6$ & $89.1$ & $89.1$ \\
    MCGC (\citeyear{pan2021multi}) & $70.9$ & $76.6$ & $91.5$ & $91.6$ & $72.2$ & $77.5$ & $90.4$ & $89.8$ \\
    MVGC (\citeyear{XiaWYGHG22}) & $62.4$ & $61.4$ & $83.1$ & $82.2$ & $47.3$ & $45.4$ & $71.8$ & $69.4$ \\
    DuaLGR (\citeyear{ling2023dual}) & $72.0$ & $78.0$ & $92.1$ & $92.1$ & $74.6$ & $80.7$ & $92.0$ & $91.4$ \\
    BTGF (\citeyear{qian2024upper}) & $75.8$ & $80.9$ & $93.2$ & $93.3$ & $62.4$ & $59.7$ & $83.1$ & $83.8$ \\
    SMHGC (ours) & $\mathbf{81.1 \pm 4.1}$ & $\mathbf{83.2 \pm 5.2}$ & $\mathbf{93.9 \pm 2.0}$ & $\mathbf{93.9 \pm 2.0}$ & $\mathbf{76.2 \pm 0.8}$ & $\mathbf{81.9 \pm 0.2}$ & $\mathbf{92.4 \pm 0.2}$ & $\mathbf{91.8 \pm 0.2}$ \\
    \midrule[1pt]
    Methods / Datasets & \multicolumn{4}{c|}{Texas ($hr$ $0.09$ \& $0.09$)} & \multicolumn{4}{c}{Chameleon ($hr$ $0.23$ \& $0.23$)} \\
    \midrule
    VGAE (\citeyear{GAE}) & $12.7 \pm 4.4$ & $21.7 \pm 8.4$ & $55.3 \pm 1.8$ & $29.5 \pm 3.1$ & $15.1 \pm 0.7$ & $12.4 \pm 0.6$ & $35.4 \pm 1.0$ & $29.6 \pm 1.7$ \\
    O2MAC (\citeyear{o2multi}) & $8.7 \pm 0.8$ & $14.6 \pm 1.8$ & $46.7 \pm 2.4$ & $29.1 \pm 2.4$ & $12.3 \pm 0.7$ & $8.9 \pm 1.2$ & $33.5 \pm 0.3$ & $28.6 \pm 0.2$ \\
    MvAGC (\citeyear{lin2021graph}) & $5.4 \pm 2.8$ & $1.1  \pm 4.1$ & $54.3 \pm 2.6$ & $19.8 \pm 5.1$ & $10.8 \pm 0.8$ & $3.3 \pm 1.7$ & $29.2 \pm 0.9$ & $24.3 \pm 0.5$ \\
    MCGC (\citeyear{pan2021multi}) & $12.7 \pm 2.9$ & $12.9 \pm 3.8$ & $51.9 \pm 0.9$ & $32.5 \pm 1.8$ & $9.5 \pm 1.3$ & $5.9 \pm 2.7$ & $30.0 \pm 2.0$ & $19.1 \pm 0.8$ \\
    MVGC (\citeyear{XiaWYGHG22}) & $8.1 \pm 3.3$ & $7.8 \pm 3.1$ & $41.8 \pm 2.6$ & $28.4 \pm 3.1$ & $12.6 \pm 0.3$ & $5.1 \pm 0.6$ & $32.8 \pm 0.4$ & $26.9 \pm 0.5$ \\
    DuaLGR (\citeyear{ling2023dual}) & $32.6 \pm 0.5$ & $26.0 \pm 0.6$ & $54.3 \pm 0.3$ & $46.8 \pm 0.3$ & $19.5 \pm 1.0$ & $\mathbf{16.0 \pm 0.6}$ & $41.1 \pm 0.8$ & $37.7 \pm 1.5$ \\
    BTGF (\citeyear{qian2024upper}) & $22.7 \pm 0.0$ & $20.5 \pm 0.0$ & $58.5 \pm 0.0$ & $35.1 \pm 0.0$ & $17.2 \pm 0.0$ & $11.5 \pm 0.0$ & $35.8 \pm 0.0$ & $30.7 \pm 0.0$ \\
    SMHGC (ours) & $\mathbf{41.8 \pm 1.1}$ & $\mathbf{46.9 \pm 3.2}$ & $\mathbf{71.3 \pm 0.8}$ & $\mathbf{49.8 \pm 2.3}$ & $\mathbf{20.0 \pm 1.3}$ & $15.1 \pm 1.0$ & $\mathbf{42.1 \pm 0.8}$ & $\mathbf{41.3 \pm 0.9}$ \\
    \bottomrule[1.5pt]
    \end{tabular}}    
\end{table*}
\begin{figure*}
    \centering
    \subfigure[NMI\%.]{
        \includegraphics[width=1.35in]{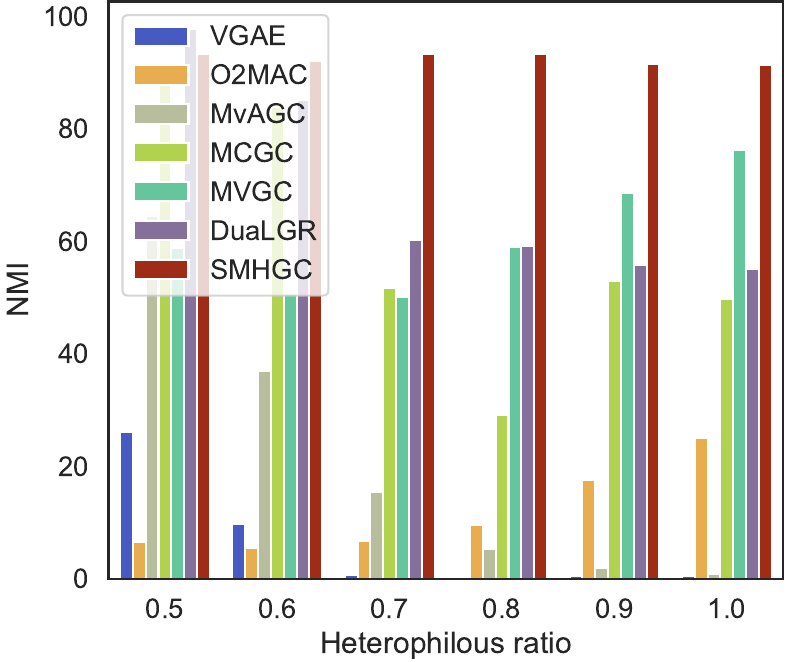}\label{fig:resultsheternmi}}  
    \hspace{0.1in}    
    \subfigure[ACC\%.]{
        \includegraphics[width=1.35in]{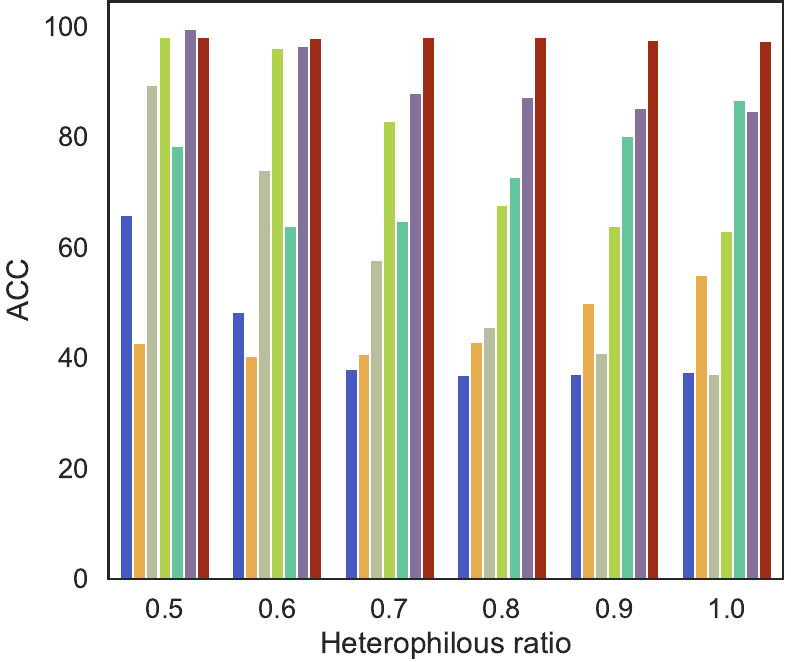}\label{fig:resultsheteracc}}
    \hspace{0.1in}
    \subfigure[$order$ and $k$.]{
        \includegraphics[width=1.5in]{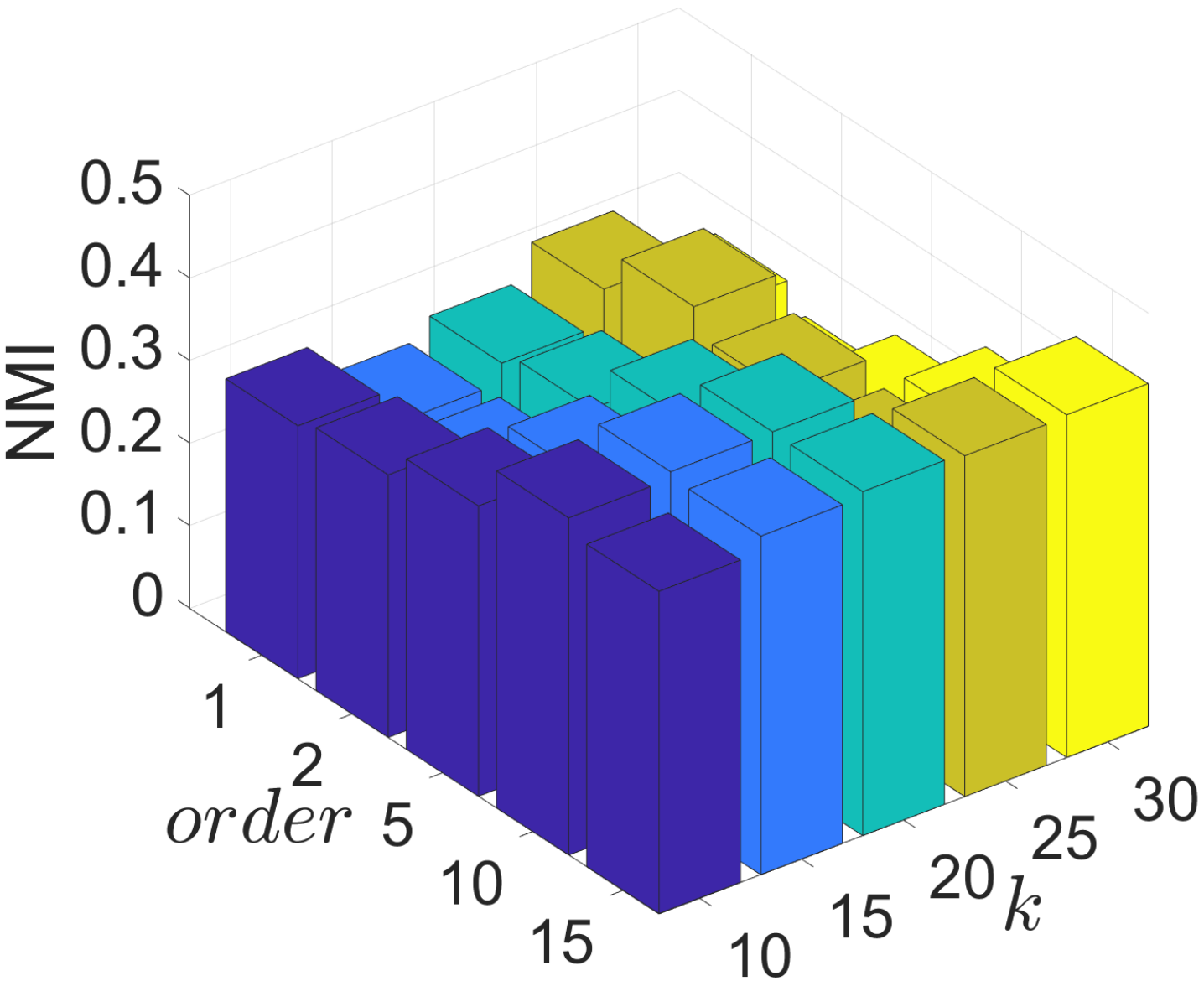}\label{fig:senodk}}
    \hspace{0.1in}
    \subfigure[$\gamma_{sim}$ and $\gamma_{r}$.]{
        \includegraphics[width=1.5in]{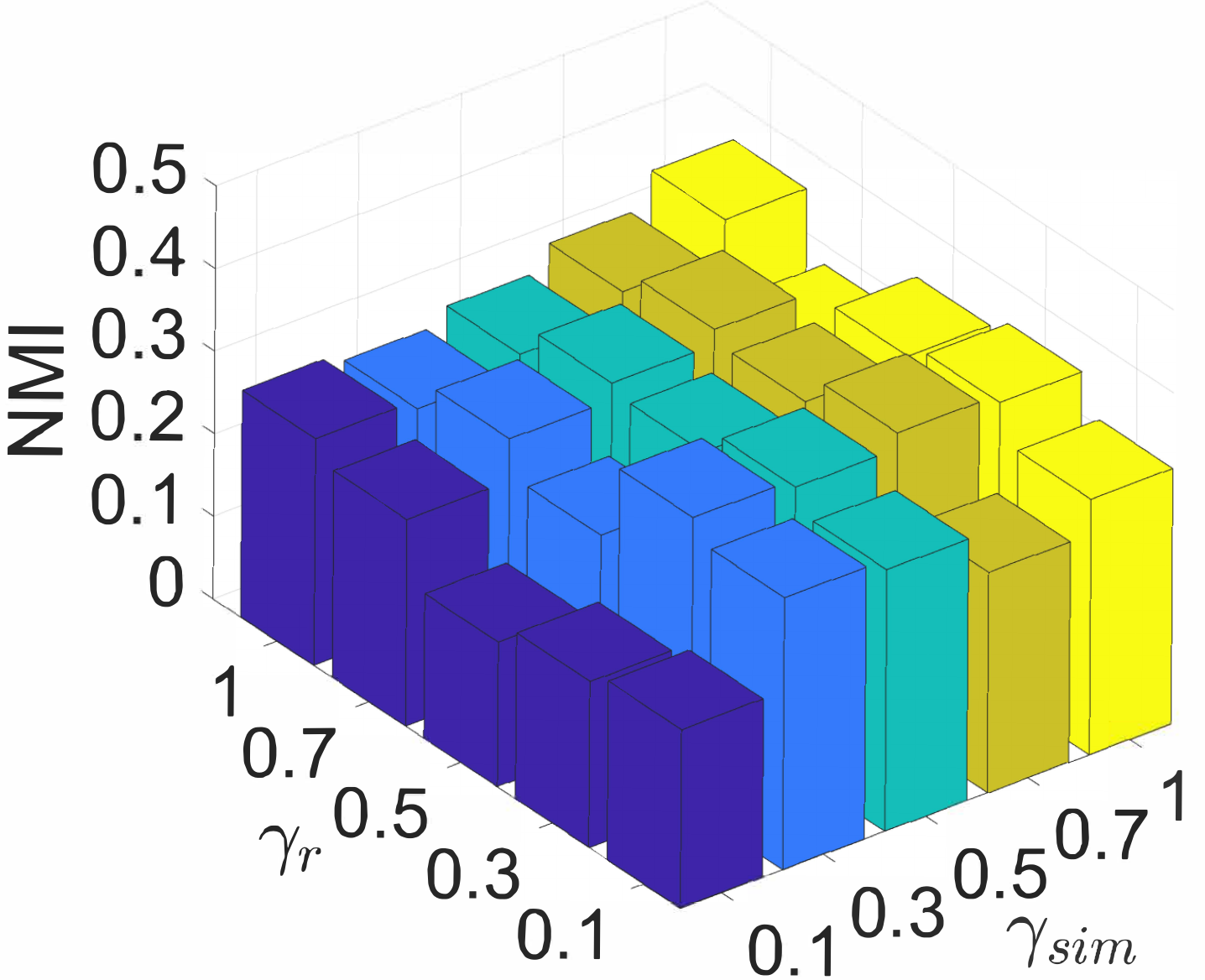}\label{fig:senheat}}
    \caption{Clustering results on six semi-synthetic ACM datasets with different heterophilous ratios (Fig.~\ref{fig:resultsheternmi} and Fig.~\ref{fig:resultsheteracc}), and parameter sensitive analysis \textit{w.r.t.} $order$, $k$ (Fig.~\ref{fig:senodk}), $\gamma_{sim}$ and $\gamma_r$ (Fig.~\ref{fig:senheat}). The whole results can be found in Appendix.}
\end{figure*}

\subsection{Overall Performance}
Table~\ref{tab:overall_results} presents the clustering performance of all compared methods on four real-world graph datasets. Notably, SMHGC demonstrates competitive ability with the baselines on homophilous datasets, including ACM and DBLP. Specifically, our method outperforms the best baseline DuaLGR on ACM, with NMI, ARI, ACC, and F1 improving by $12.6\%$, $6.7\%$, $2.0\%$, and $2.0\%$, respectively. On two heterophilous datasets, \textit{i.e.}, Texas and Chameleon, our method exhibits excellent results compared to other traditional methods that rely on the homophily assumption. For instance, considering Texas, whose homophilous ratio $hr$ is only $0.09$, the best result from VAGE achieves only $55.3\%$ accuracy, while SMHGC achieves significantly higher accuracy, reaching up to $71.3\%$. This underscores the outstanding performance of SMHGC on heterophilous graphs, leveraging similarity. It also highlights the limitations of existing traditional methods when confronted with heterophilous graphs. 
Additionally, Fig.~\ref{fig:resultsheternmi} and Fig.~\ref{fig:resultsheteracc} illustrate the partial clustering performance of SMHGC and other baselines on the six semi-synthetic ACM datasets with varying heterophilous ratios ranging from $0.5$ to $1.0$. The performance of these baselines deteriorates as the heterophilous ratio increases, indicating a decline in performance with the increase in heterophilous information. Although O2MAC and MVGC display a relatively opposite trend, their performance also significantly decreases compared to when dealing with homophilous graphs. This suggests that heterophilous graphs pose a challenge to existing MVGC methods. In contrast, leveraging similarity, SMHGC maintains relatively stable performance despite the increasing heterophilous ratio. This stability can be attributed to the extraction of homophilous information coupled with similarities, as well as intra- and inter-view fusion based on global similarity and consensus, empirically demonstrating the robustness of similarity for MVHGC.

\begin{table}[!t]
\small
    \centering
    \caption{The ablation study of SMHGC on ACM and Texas.}
    \label{tab:ablation}
    \begin{tabular}{c|cccc|cccc}
    \toprule[1.5pt]
    \multirow{2}*{Compenents} & \multicolumn{4}{c|}{ACM ($hr$ $0.82$ \& $0.64$)} & \multicolumn{4}{c}{Texas ($hr$ $0.09$ \& $0.09$)} \\
         & NMI\% & ARI\% & ACC\% & F1\% & NMI\% & ARI\% & ACC\% & F1\% \\
    \midrule
    w/o $\mathcal{L}_{sim}$ & $52.7$ & $48.3$ & $66.5$ & $65.8$ & $20.2$ & $31.7$ & $61.8$ & $36.2$ \\
    w/o $\mathcal{L}_r$  & $48.1$ & $46.9$  & $72.0$ & $71.7$ & $38.2$ & $41.0$ & $67.8$ & $49.1$ \\
    w/o $\mathcal{L}_{kl}$  & $74.9$ & $76.3$  & $91.4$ & $91.4$ & $40.5$ & $46.7$ & $71.0$ & $49.0$ \\
    \midrule
    w/o $\mathbf{A}_x^v$ & $39.1$  & $35.2$ & $66.4$ & $66.0$  & $22.3$  & $26.9$ & $60.1$ & $34.5$ \\
    w/o $\mathbf{A}_a^v$ & $70.9$  & $70.2$ & $88.1$ & $87.0$  & $37.6$ & $41.1$ & $68.9$ & $47.7$ \\
    \midrule
    w/o $\omega_x, \omega_a$ & $81.0$  & $82.8$ & $93.5$ & $93.3$  & $38.7$  & $46.7$ & $69.9$ & $49.2$ \\
    \midrule
    \textbf{SMHGC} & $\mathbf{81.1}$ & $\mathbf{83.2}$ & $\mathbf{93.9}$ & $\mathbf{93.9}$ & $\mathbf{41.8}$ & $\mathbf{46.9}$ & $\mathbf{71.3}$ & $\mathbf{49.8}$ \\
    \bottomrule[1.5pt]
    \end{tabular} 
    
\end{table}

\subsection{Ablation Study}
\paragraph{\textbf{Effect of each loss.}}
As depicted in Table~\ref{tab:ablation}, the performance of SMHGC experiences a significant drop in the absence of $\mathcal{L}_{sim}$. This not only validates Proposition~\ref{good} empirically but also underscores the feasibility and effectiveness of exploring homophilous information in node features and neighbor patterns through the proposed similarity loss. Additionally, $\mathcal{L}_r$ represents the reconstruction loss of the node features, aiming to retain as much key information as possible. Its absence leads to a degradation in the model's performance across all metrics. On the other hand, $\mathcal{L}_{kl}$ contributes to obtaining distinguishable embeddings, although its impact appears to be subtle.
\paragraph{\textbf{Effect of couple similarities.}}
As observed in the third and fourth rows of Table~\ref{tab:ablation}, the absence of either $\mathbf{A}_x^v$ or $\mathbf{A}_a^v$ results in a certain degree of performance degradation. This indicates that both node features and neighbor patterns contain certain complementary homophilous information. Consequently, it underscores the necessity of mining homophilous information from node features and neighbor patterns, respectively.
\paragraph{\textbf{Effect of global similarity.}}
The absence of $\omega_x$ and $\omega_a$ have a negligible effect on the model's performance on ACM and a relatively strong effect on Texas. This discrepancy may arise from the varying relevance of node features and neighbor patterns to the downstream task across different datasets.

\subsection{Parameter Sensitive Analysis}
SMHGC primarily relies on two key hyperparameters: $k$ and $order$, which determine the number of edges in $\mathbf{S}^v$ and the aggregation order of $\mathbf{S}^v$, respectively. To explore their specific effects on the model, we conduct a parameter sensitivity analysis of these two hyperparameters on ACM and Texas datasets. Part of the results is depicted in Fig.~\ref{fig:senodk}. 
As illustrated, our SMHGC has better performance when $k$ is in different ranges depending on the datasets. This observation suggests that the model aggregates an optimal amount of homophilous messages from the neighborhood when $k$ is appropriately chosen. Empirically, selecting $k$ within $8\%$ to $12\%$ of the number of nodes appears to yield better results. Additionally, as $order$ increases, SMHGC exhibits a trend of initially rising and then stabilizing, with the peak performance observed within the range of $2$ to $10$. 
Furthermore, Fig.~\ref{fig:senheat} illustrates the effect of different values of $\gamma_{sim}$ and $\gamma_r$ on the final results. Overall, the varying weights of $\mathcal{L}_{sim}$ and $\mathcal{L}_{r}$ have a relatively moderate impact on the model outcomes. Increasing the weight of $\mathcal{L}_{sim}$ slightly improves the final results. 
In this study, both $\gamma_{sim}$ and $\gamma_r$ are set to 1 according to the experimental results.

\section{Conclusion}
In this study,we analyzed the observation about homophily and similarity, and introduce an effective solution for multi-view heterophilous graph clustering, called SiMilarity-enhanced homophily for Multi-view Heterophilous Graph Clustering (SMHGC). Confronted with the challenges posed by heterophilous graphs, we empirically demonstrated the robust power of similarity for unsupervised clustering tasks. Our analysis explores how the similarity could enhance homophilious and clustering performance. Constructed on this foundation, we propose two regularization losses, \emph{i.e.}, neighbor pattern similarity and node feature similarity,  to enhance graph homophily under the guidance of introduced multi-view global similarity. Further, we propose a paradigm for fusing inter- and intra-view information, enabling the integration of homophilous information from different sources and levels through the utilization of global similarity and multi-view consensus. Extensive experiments demonstrate the strong robustness of SMHGC for multi-view heterophilous graph clustering, validating the feasibility and effectiveness of our proposed solution.




\clearpage

\bibliography{main}
\bibliographystyle{tmlr}

\clearpage

\appendix
\section{RELATED WORKS}
\subsection{Multi-View Graph Clustering}
With the advancement of GNNs, researchers are eager to explore the graph structural information for multi-view clustering. In recent years, an abundance of methods for MVGC have emerged. O2MAC~\citep{o2multi} pioneered the application of GNNs in MVGC. Their approach is to encode multi-view graphs into a lower-dimensional space using a single-view graph convolutional encoder and a multi-view graph structure decoder. \citet{MVGCC} proposed two-pathway graph encoders, which facilitate the mapping of graph embedding features and the acquisition of view-consistent information. \citet{hassani2020contrastive} proposed an innovative GNN-based solution aimed at acquiring node and graph level representations specifically for multi-view self-supervised learning. \citet{pan2021multi} leveraged the technique of contrastive learning to excavate the shared geometric and semantic patterns, thereby facilitating the learning of a consensus graph. MVGC~\citep{XiaWYGHG22} systematically explored the cluster structure by employing a graph convolutional encoder, which is trained to learn the self-expression coefficient matrix. However, these methods demonstrate a high sensitivity to the homophily of graphs. 
Furthermore, DuaLGR, proposed by~\cite{ling2023dual}, initially attempted to address the MVHGC problem. Despite its promising performance on heterophilous graphs, it struggles to explain the rationale of its performance enhancement, which limits its interpretability and generalizability. In contrast, our approach explores the relationship between similarity and homophilous information, providing reliable support for the MVHGC problem from a data-centric view. 

\subsection{Heterophilous Graph Learning}
Many works concentrate on single-view heterophilous graph learning. For instance, \citet{li2022restructuring} proposed a graph restructuring method that enhances spectral clustering by aligning with node labels. A feature extraction technique adaptable to both homophilous and heterophilous graphs devised by \citet{chanpuriya2022simplified}, demonstrating its effectiveness in node classification. These methods reduce the impact of heterophily but are challenging to apply to MVHGC due to their reliance on labels in supervised tasks.
Additionally, unsupervised graph learning methods, such as GREET \citep{liu2023beyond}, which uses an edge discriminator to differentiate between homophilous and heterophilous edges, and \citet{xiao2022decoupled} developed the decoupled self-supervised learning (DSSL) framework. While these approaches can handle heterophilous graphs, they are complex and often designed for single-view tasks, making them difficult to generalize to multi-view settings. Effective solutions for multi-view heterophilous graphs remain limited.

\section{Proof of Theorem 1}
\begin{theorem}
Let $\mathcal{L}(h)$ be the expectation of $\hat{\mathcal{L}}(h)$. Suppose that for any $a \in \mathcal{A}$ and $h \in \mathcal{H}$, there exists $M < \infty$ such that $\Vert aa^\mathrm{T} \Vert$, $\Vert h_a(a)h_a(a)^\mathrm{T} \Vert \in [0, M]$ hold. Then with probability $1-\tau$ for any $h \in \mathcal{H}$ the inequality holds:
\begin{equation}
    \hat{\mathcal{L}}(h) \leq \mathcal{L}(h) + 2\sqrt{2}M^2\sqrt{N}(4 + 3\sqrt{\log\frac{1}{\tau}}).
\end{equation}
\end{theorem}

\begin{proof}
For the given neighbor pattern set $P = \{ a_1, a_2, \cdots, a_N \}$, let $P^{\prime}$ be the neighbor pattern set where only one neighbor pattern $\overline{a}_r$ differs from $P$, and let $\hat{\mathcal{L}}^{\prime}(h)$ represent the empirical risk of $h$ on $P^{\prime}$. We can have:
\begin{equation}
\begin{aligned}
    &\vert \sup_{h \in \mathcal{H}} \vert \hat{\mathcal{L}}(h) - \mathcal{L}(h)\vert - \sup_{h \in \mathcal{H}} \vert \hat{\mathcal{L}}^{\prime}(h) - \mathcal{L}(h) \vert\vert \\
    &\leq \sup_{h \in \mathcal{H}} \vert \hat{\mathcal{L}}(h) - \hat{\mathcal{L}}^{\prime}(h)\vert \\
    &= \sup_{h \in \mathcal{H}} \vert \frac{2}{N} \sum_{j=1}^N \Vert h_a(a_r)h_a(a_j)^\mathrm{T} - a_r a_j^\mathrm{T} \Vert^2 \\
    &\quad - \Vert h_a(\overline{a}_r)h_a(a_j)^\mathrm{T} - \overline{a}_r a_j^\mathrm{T} \Vert^2 \vert \\
    &= \sup_{h \in \mathcal{H}} \frac{2}{N} \vert \sum_{j=1}^N \Vert h_a(a_r)h_a(a_j)^\mathrm{T} \Vert^2 \\
    &\quad - \Vert h_a(\overline{a}_r)h_a(a_j)^\mathrm{T} \Vert^2 + \Vert a_r a_j^\mathrm{T} \Vert^2 - \Vert \overline{a}_r a_j^\mathrm{T} \Vert^2 \\
    &\quad + 2 \Vert h_a(a_r)h_a(a_j)^\mathrm{T} \Vert \Vert a_r a_j^\mathrm{T} \Vert \\
    &\quad + 2 \Vert h_a(\overline{a}_r)h_a(a_j)^\mathrm{T} \Vert \Vert \overline{a}_r a_j^\mathrm{T} \Vert \vert \\
    &\leq 12M^2.
\end{aligned}
\end{equation}

In order to bound the expectation term $\mathbb{E}_P \sup_{h \in \mathcal{H}}\vert \mathcal{L}(h) - \hat{\mathcal{L}}(h) \vert$, let $\sigma_1$, $\sigma_2$, $\cdots$, $\sigma_N$ be \textit{i.i.d.} Rademacher random variables, we can have:
\begin{equation}\label{eq:expect}
    \begin{aligned}
        &\mathbb{E}_P \sup_{h \in \mathcal{H}}\vert \mathcal{L}(h) - \hat{\mathcal{L}}(h) \vert \\
        &\leq \mathbb{E}_{P,P} \sup_{h \in \mathcal{H}} \vert \frac{1}{N} \sum_{i=1}^N\sum_{j=1}^N \Vert h_a(a_i)h_a(a_j)^\mathrm{T} \\
        &\quad - a_i a_j^\mathrm{T} \Vert^2 - \Vert h_a(\hat{a}_i)h_a(\hat{a}_j)^\mathrm{T} - \hat{a}_i \hat{a}_j^\mathrm{T} \Vert^2 \vert \\
        &=2\mathbb{E}_{P,\sigma} \sup_{h \in \mathcal{H}} \vert \frac{2}{N} \sum_{i=1}^N \sigma_i \sum_{j=1}^N \Vert h_a(a_r)h_a(a_j)^\mathrm{T} - a_r a_j^\mathrm{T} \Vert^2 \vert.
    \end{aligned}
\end{equation}

According to Khintchine-Kahane inequality, Eq.~(\ref{eq:expect}) can be bounded as:
\begin{equation}\label{eq:expleq}
\begin{aligned}
    &2\mathbb{E}_{P,\sigma} \sup_{h \in \mathcal{H}} \vert \frac{2}{N} \sum_{i=1}^N \sigma_i \sum_{j=1}^N \Vert h_a(a_r)h_a(a_j)^\mathrm{T} - a_r a_j^\mathrm{T} \Vert^2 \vert \\
    &\leq 2\mathbb{E}_{P,\sigma} \sup_{h \in \mathcal{H}} (\frac{2}{N} \sum_{i=1}^N (\sum_{j=1}^N \Vert h_a(a_r)h_a(a_j)^\mathrm{T} - a_r a_j^\mathrm{T} \Vert^2)^2 )^{\frac{1}{2}}\\
    &\leq 8\sqrt{2}M^2\sqrt{N}.
\end{aligned}
\end{equation}

Substituting Eq.~(\ref{eq:expleq}) into Eq.~(\ref{eq:expect}), we have:
\begin{equation}
    \mathbb{E}_P \sup_{h \in \mathcal{H}}\vert \mathcal{L}(h) - \hat{\mathcal{L}}(h) \vert \leq 8\sqrt{2}M^2\sqrt{N}.
\end{equation}

Finally, according to the McDiarmid inequality, it can be conclued that with probability $1-\tau$ for any $h \in \mathcal{H}$ the following inequality holds:
\begin{equation}
\begin{aligned}
    &\hat{\mathcal{L}}(h) \leq \mathcal{L}(h) + 8\sqrt{2}M^2\sqrt{N} + 6\sqrt{2}M^2\sqrt{\log\frac{1}{\tau}}\sqrt{N}\\
    &=\mathcal{L}(h) + 2\sqrt{2}M^2\sqrt{N}(4 + 3\sqrt{\log\frac{1}{\tau}}).
\end{aligned}
\end{equation}

For the similarity loss from node features $\mathcal{L}_{sim_x}$, the proof and derivation procedure is the same as above.
\end{proof}

\section{Complexity Analysis}
Considering $N$, $K$, $V$, $d$ and $order$ as the number of nodes, the number of cluster centroids, the number of views, the maximum dimensionality of input $\mathbf{X}^v$ and the order of aggregation. Let $L$ denote the maximum number of neurons in MLP's hidden layers (the encoders and decoders are instantiated with MLPs in this work), and $d_h$ as the dimensionality of the embedding ($\mathbf{H}^v$, $\overline{\mathbf{H}}$). The complexity of the encoders and decoders ($f_a^v$, $f_x^v$, $f^v_{\phi^v}$ and $g^v_{\xi^v}$) from the V views is $O(VNL^2)$. The similarity calculated in Eq.~(5) needs $O(VN^2)$. The complexity of the aggregation process in Eq.~(10) is $O(VdN^2order)$, which is depend on the aggregation order $order$. In addition, the evaluation function (instantiated with $K$-means) in Eq.~(11) needs $O(VNKd_h)$. In summary, the complexity of SMHGC is $O(VN(L^2 + Kd_h) + dVN^2order)$, which is proportional to the square of the number of nodes $N^2$ and the aggregation order $order$.

\begin{table*}[!ht]
    \centering
    \caption{The statistics information of the four graph datasets. $hr$ is the homophilous ratio, which describes the ratio of \#Homo-edges and \#Edges.}\label{tab:datasets}
    \begin{tabular}{cccccccc}
    \toprule
        Datasets & \#Clusters & \#Nodes & \#Features & \#Graphs & \#Homo-edges & \#Edges & $hr$ \\
        \midrule
        \multirow{2}*{ACM} & \multirow{2}*{$3$} & \multirow{2}*{$3025$} & $1870$ & $\mathcal{G}^1$ & $21550$ & $26252$ & $0.82$ \\
         & & & $1870$ & $\mathcal{G}^2$ & $1411658$ & $2207736$ & $0.64$ \\
         \midrule
         \multirow{3}*{DBLP} & \multirow{3}*{$4$} & \multirow{3}*{$4057$} & $334$ & $\mathcal{G}^1$ & $5636$ & $7056$ & $0.80$ \\
         & & & $334$ & $\mathcal{G}^2$ & $3346042$ & $4996438$ & $0.67$ \\
         & & & $334$ & $\mathcal{G}^3$ & $2183134$ & $6772278$ & $0.32$ \\
         \midrule
         \multirow{2}*{Texas} & \multirow{2}*{$5$} & \multirow{2}*{$183$} & $1703$ & $\mathcal{G}^1$ & $50$ & $574$ & $0.09$\\
         & & & $1703$ & $\mathcal{G}^2$ & $50$ & $574$ & $0.09$ \\
        \midrule
         \multirow{2}*{Chameleon} & \multirow{2}*{$5$} & \multirow{2}*{$2277$} & $2325$ & $\mathcal{G}^1$ & $14476$ & $62792$ & $0.23$\\
         & & & $2325$ & $\mathcal{G}^2$ & $14476$ & $62792$ & $0.23$\\
    \bottomrule
    \end{tabular}
    
\end{table*}

\begin{figure*}[!ht]
    \centering
    \subfigure[NMI\%.]{
        \includegraphics[width=1.35in]{nmi_sy.pdf}}
    \hspace{0.1in}
    \subfigure[ARI\%.]{
        \includegraphics[width=1.35in]{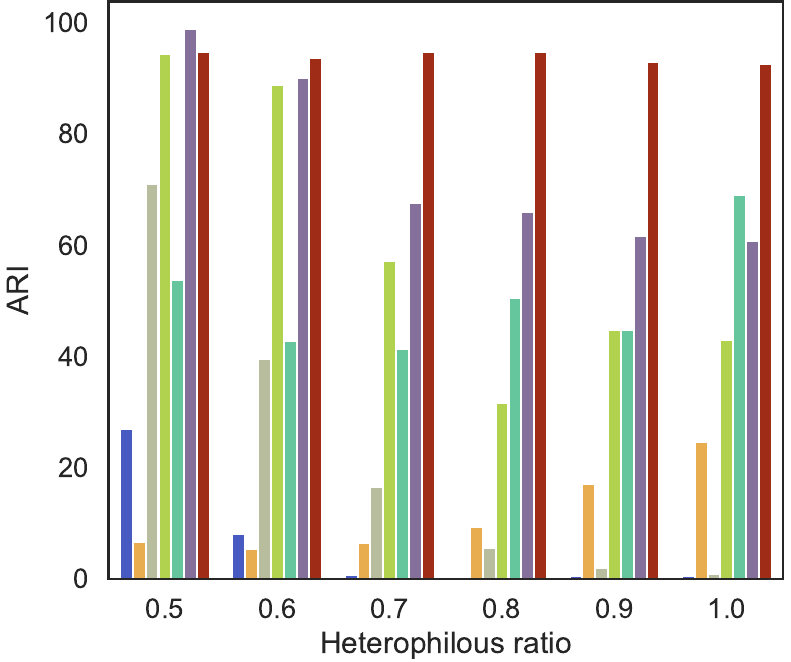}}    
    \hspace{0.1in}    
    \subfigure[ACC\%.]{
        \includegraphics[width=1.35in]{acc_sy.pdf}}
    \hspace{0.1in}
    \subfigure[F1\%.]{
        \includegraphics[width=1.35in]{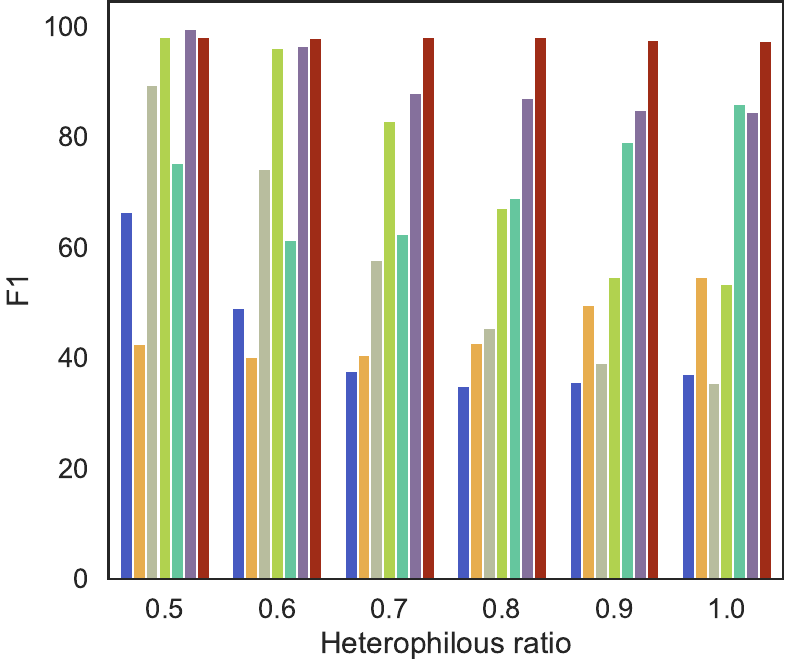}}
    \caption{Clustering results on six semi-synthetic ACM datasets with different heterophilous ratios.}
    \label{fig:performanceHetermore}
\end{figure*}

\begin{figure*}[!ht]
    \centering
    \subfigure[ACC on ACM.]{
        \includegraphics[width=1.35in]{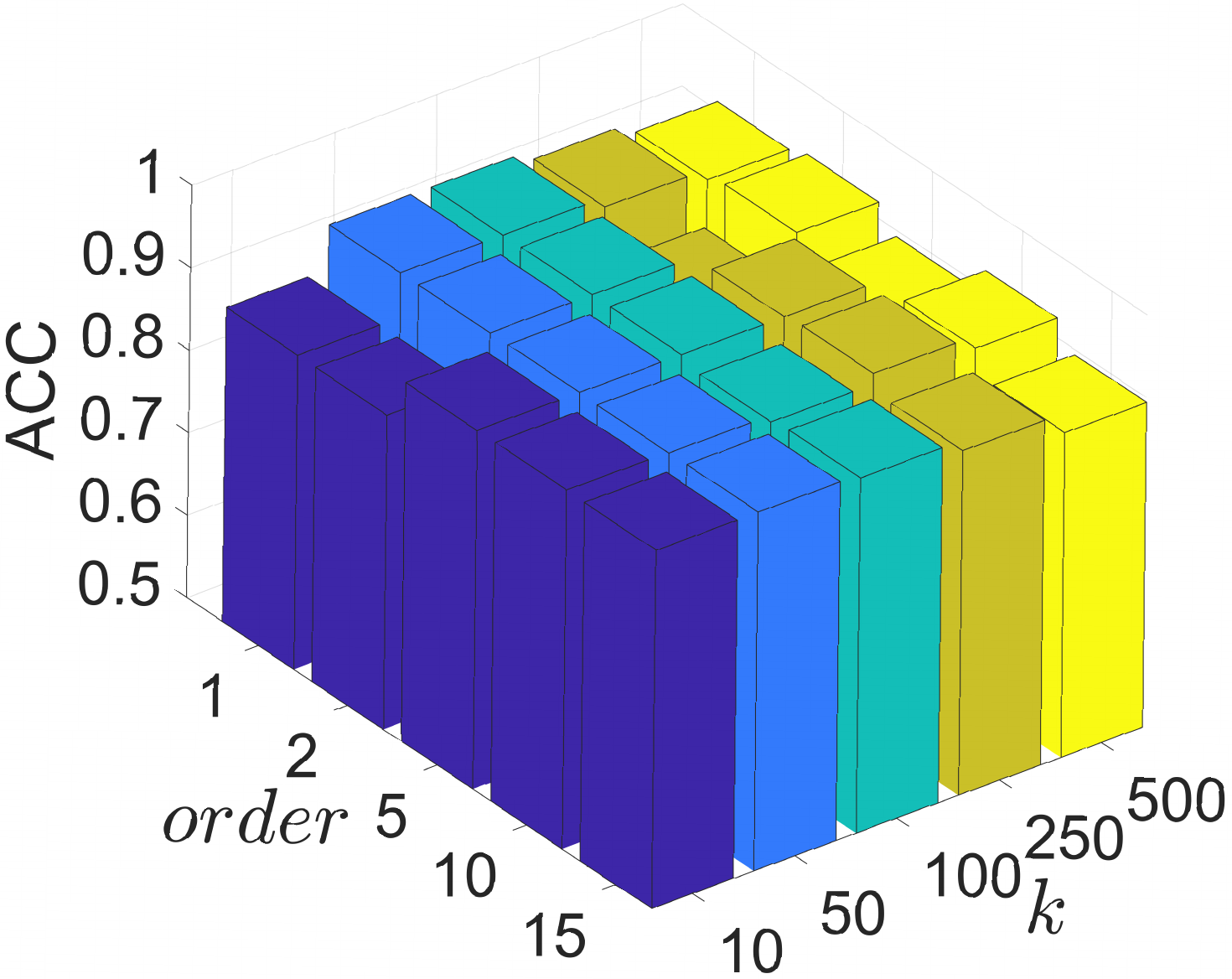}}
    \hspace{-0.1in}
    \subfigure[NMI on ACM.]{
        \includegraphics[width=1.35in]{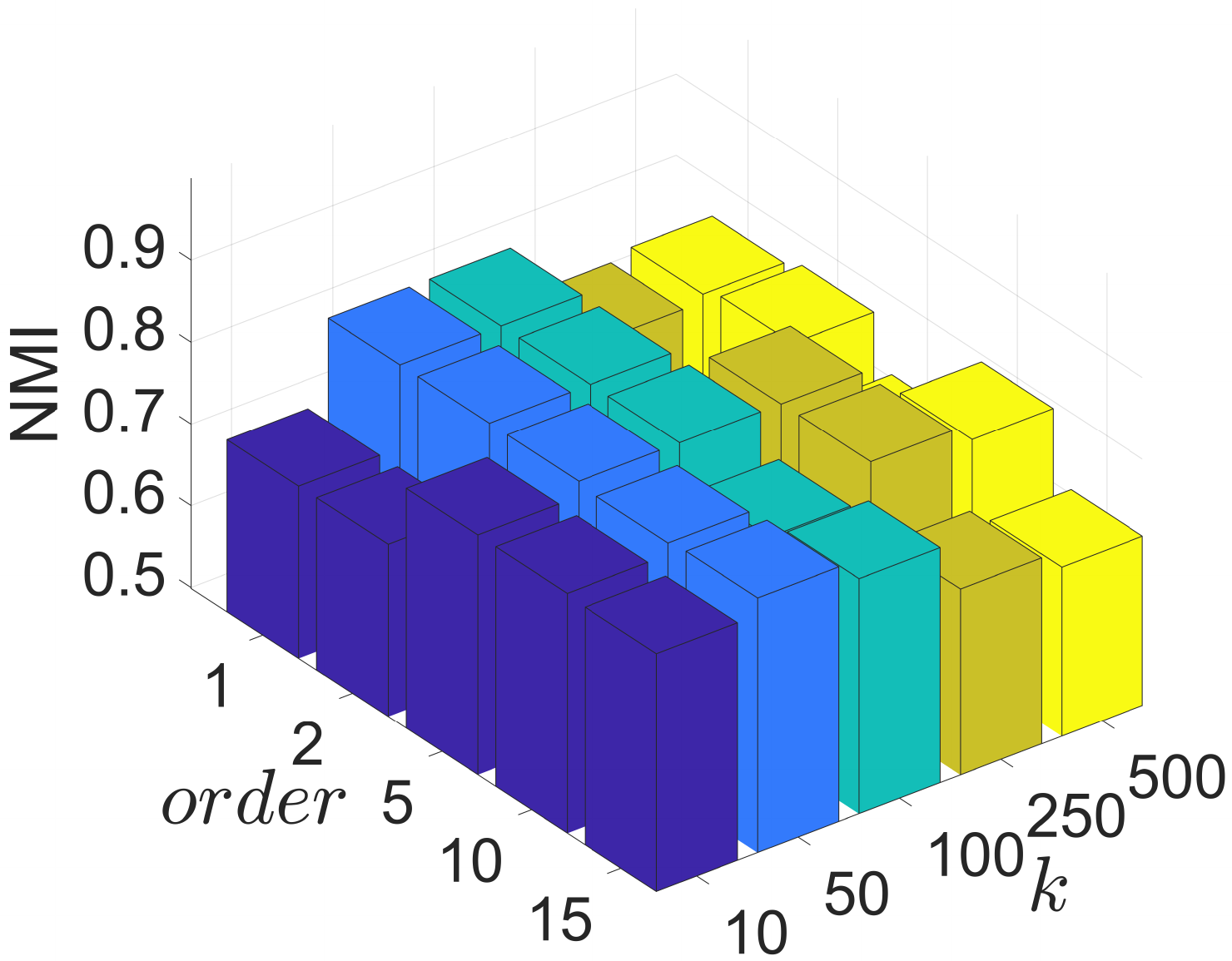}}    
    \hspace{-0.1in}    
    \subfigure[ACC on Texas.]{
        \includegraphics[width=1.35in]{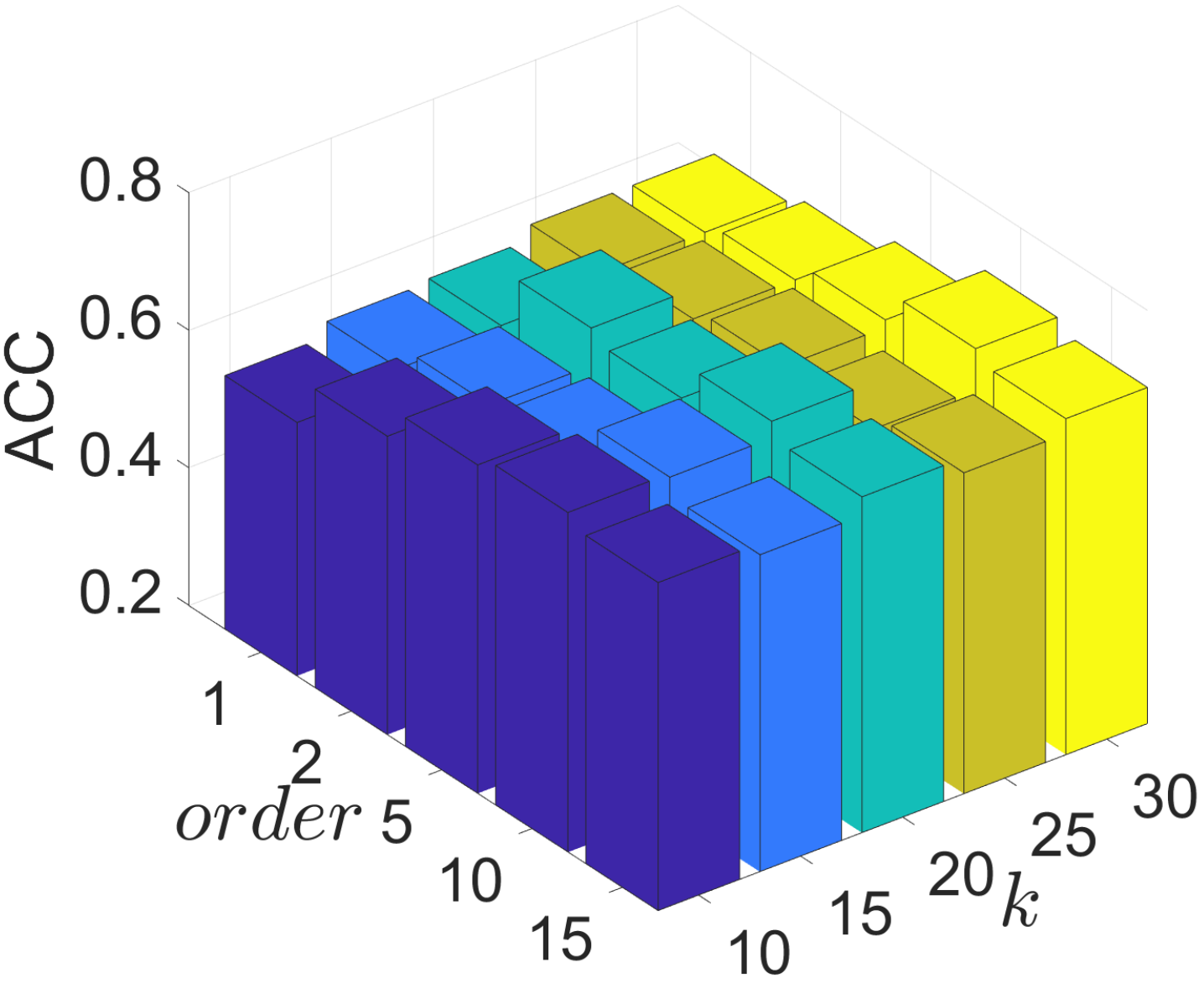}}
    \hspace{-0.1in}
    \subfigure[NMI on Texas.]{
        \includegraphics[width=1.35in]{texas_nmi_sen.pdf}}
    \caption{Sensitive analysis of ACC and NMI on ACM and Texas with $order$ and $k$.}
    \label{fig:sensitivemore}
\end{figure*}

\section{Details of experiments}
\label{app:exp}

\subsection{Datasets}
Both homophilous and heterophilous graph datasets are considered in this work to evaluate our model. Besides, six datasets were synthesized to obtain a smoother and more intuitive evaluation regarding to homophily ratio ($hr$). The statistics information of these datasets is summarized in Table~\ref{tab:datasets}. Particularly, ACM~\footnote{https://dl.acm.org/} and DBLP~\footnote{https://dblp.uni-trier.de/} are two widely used homophilous multi-view graph datases, which are two citation networks that contain two and three graph respectively. Texas and Chameleon are two widely used heterophilous grah datasets, which are a webpage graph from WebKB~\footnote{http://www.cs.cmu.edu/afs/cs.cmu.edu/project/theo-11/www/wwkb} and a subset of the Wikipedia network~\cite{rozemberczki2021multi}. As Texas and Chameleon are single view graph data, we duplicate the feature and graph as a second view.
In addition, we take ACM as an example to generate different $hr$ graphs with random sampling, each view's graph with the same number of edges as the original view for evaluation, whose $hr$ range from $0.0$ to $0.5$ with the step of $0.1$~\cite{ling2023dual}.

\subsection{Implementation Details}
The experiments are conducted on a Windows machine with a NVIDIA GeForce RTX 2060 GPU and Intel(R) Core(TM) i5-9400F CPU @ 2.90GHz, where the version of CUDA is 11.4 and the version of torch is 1.10.2. The implementation of SMHGC will be published.

\section{More experiment results}\label{app:moreresults}
\subsection{Complete Clustering Results about Clustering Results on Six Semi-synthetic ACM Datasets}
The whole results about clustering results on six semi-synthetic ACM datasets with different heterophilous ratios are shown in Fig.~\ref{fig:performanceHetermore}. These results suggest that heterophilous graphs are a barrier to the existing MVGC methods, while our SMHGC still has a relatively smooth performance with the power of similarity.

\subsection{Complete Clustering Results about Parameter Sensitive Analysis}
The whole results about the parameter sensitive analysis of two hyperparameters, \textit{i.e.} $k$ and $order$, on ACM and Texas are shown in Fig.~\ref{fig:sensitivemore}. Specifically, SMHGC seems to have a stable performance in terms of ACC on both ACM and Texas, where ACC on ACM is mainly between $0.85$ and $0.95$ and on Texas is mainly between $0.55$ and $0.75$. The NMI, on the other hand, is relatively more volatile, with the NMI on ACM fluctuating mainly between $0.65$ and $0.80$, and on Texas lying mainly between $0.25$ and $0.40$.

\end{document}